\documentclass{article}

\PassOptionsToPackage{numbers, compress}{natbib}

\usepackage[preprint]{neurips_2026}

\usepackage[utf8]{inputenc}
\usepackage[T1]{fontenc}
\usepackage{microtype}
\usepackage{graphicx}
\usepackage{subcaption}
\usepackage{booktabs}
\usepackage{nicefrac}
\usepackage{url}

\usepackage{hyperref}

\usepackage{algorithm}
\usepackage{algorithmic}

\usepackage{amsmath}
\newcommand{\RETURN}{\STATE \textbf{return} }
\usepackage{amssymb}
\usepackage{amsfonts}
\usepackage{mathtools}
\usepackage{amsthm}
\usepackage[table]{xcolor}
\definecolor{HiRed}{RGB}{251,233,231}   
\usepackage{multirow}
\usepackage{pifont}
\usepackage{csquotes}
\usepackage{siunitx}

\usepackage[capitalize,noabbrev]{cleveref}

\theoremstyle{plain}
\newtheorem{theorem}{Theorem}[section]

\newtheorem{lemma}[theorem]{Lemma}

\theoremstyle{definition}

\theoremstyle{remark}

\newcommand{\cmark}{\ding{51}} 
\newcommand{\xmark}{\ding{55}}

\usepackage[textsize=tiny]{todonotes}

\title{Closed-Loop Bidirectional Prompting for Adversarial Robustness of Vision Language Models}

\author{%
  Xiao Liu$^{*\,1,2}$\qquad
  Jiaxiang Liu$^{*\,2}$\qquad
  Boci Peng$^{3}$\\[2pt]
  \bf Boren Hu$^{4}$\qquad
  Yusong Wang$^{5}$\qquad
  Xiwen Chen$^{6}$\\[2pt]
  \bf Prayag Tiwari$^{7}$\qquad
  Liming Zhang$^{\dagger\,1}$\qquad
  Mingkun Xu$^{\dagger\,2}$\\[8pt]
  $^{1}$University of Macau\quad
  $^{2}$Guangdong Institute of Intelligence Science and Technology\\
  $^{3}$Peking University\quad
  $^{4}$Independent Researcher\quad
  $^{5}$Institute of Science Tokyo\\
  $^{6}$Morgan Stanley\quad
  $^{7}$Halmstad University\\[4pt]
  Correspondence: \texttt{yc57995@um.edu.mo}, \texttt{mingkunxu@gdiist.cn}
}

\begin{document}

\maketitle
\renewcommand{\thefootnote}{\fnsymbol{footnote}}
\footnotetext[1]{Equal contribution.}
\footnotetext[2]{Corresponding author.}

\begin{abstract}
Vision Language Models adapt well to downstream tasks but are highly vulnerable to adversarial perturbations that disrupt cross-modal semantic alignment.
Existing defenses are largely unidirectional or structural, failing to exploit bidirectional cross-modal complementarity and instance-wise adaptive protection.
To overcome the limitations of unidirectional and static defenses in adversarial settings, we propose Closed-Loop Bidirectional Prompting, casting robust adaptation as cross-modal agreement recovery via a dynamic feedback loop on frozen encoders.
A Semantic Anchor is introduced as a stable prior to constrain cyclic updates and mitigate perturbation-induced feature corruption.
Through anchor-based bootstrapping, textual semantics denoise visual representations, while the refined visuals enable instance-adaptive prompt updating, yielding a rectified and robust consensus.
Extensive evaluations across 11 datasets validate state-of-the-art robustness and strong base-to-new generalization, while maintaining a favorable trade-off between computational cost and accuracy.
\end{abstract}

\section{Introduction}
\label{sec:intro}

Vision Language Models (VLMs)~\cite{radford2021learning, liu2023visual, kim2021vilt, yao2021filip, li2022blip} have advanced representation learning by aligning visual and textual modalities in a joint embedding space through self-supervised pre-training~\cite{caron2021emerging, grill2020bootstrap, chen2020simple}. The resulting cross-modal alignment supports strong generalisation across a range of tasks, from classification~\cite{wang2023improving} to dense prediction such as segmentation~\cite{zhou2022extract} and object detection~\cite{zhang2023simple}, and underpins a growing set of real-world applications~\cite{zhang2025multimodal,liu2025kpl}.

The same tight visual--textual coupling that drives this performance, however, can also be a source of fragility~\cite{malik2025robust}: small input perturbations can disturb the alignment and propagate to the output. The concern is amplified in deployment, since most widely used VLMs are built on publicly released foundations such as CLIP, granting potential adversaries white-box access to the encoder and making gradient-based attacks readily available~\cite{kong2024patch, madry2017towards, carlini2017towards}. These factors motivate a closer examination of VLM robustness under adversarial settings.

Defences for VLMs fall broadly into training-time and test-time paradigms, and existing methods within each paradigm have complementary limitations. \emph{Adversarial Fine-Tuning} (AFT)~\cite{mao2022understanding, schlarmann2024robust, wang2024pre} adversarially updates the visual encoder; while effective on benchmarks such as ImageNet, it is computationally expensive and can distort the pre-trained feature manifold, as evidenced by the Centered Kernel Alignment (CKA) drift relative to the original CLIP encoder reported in Appendix~\ref{sec:cka_drift} and Figure~\ref{figs:cka_drift}, which sometimes hurts clean-data generalisation. \emph{Adversarial Prompt Tuning} (APT)~\cite{li2024one, zhou2022learning, zhang2024adversarial} avoids modifying the backbone by optimising a small set of soft prompts, but the resulting prompts are static and shared across all inputs, which constrains adaptation to instance-specific perturbations and tends to overfit under few-shot supervision. \emph{Test-time defences} avoid this overhead by adapting per-instance during inference~\cite{xing2025clip, liu2025self}, but most focus on attack-side characteristics (e.g., perturbation magnitude) rather than the underlying cross-modal alignment.

A common thread across these paradigms is that they adjust a single modality—either the image branch (Figure~\ref{fig:advvlp}) or the text branch—while treating the other as fixed context. Even APT methods that bridge modalities typically do so as a unidirectional projection (Figure~\ref{fig:maple}), without an iterative exchange that lets each branch correct the other. Since adversarial attacks on VLMs are designed precisely to \emph{decouple} the two branches, a defence that performs \emph{bidirectional} recovery---where visual and textual features actively interact to re-establish their alignment---is a natural fit for the threat model.

Building on this observation, we propose \textbf{Closed-Loop Bidirectional Prompting (CLBP)}, a defence that establishes a feedback mechanism between the two branches while keeping the pre-trained encoders frozen (Figure~\ref{fig:clbp}). CLBP couples two lightweight adapters. \emph{Text-to-Vision (T2V) denoising}: a T2V adapter consumes the current text embeddings and emits visual prompts that are prepended to the image patches, so that adversarial noise in the visual feature is filtered through cross-attention against the relatively stable language signal. Unlike VPT~\cite{jia2022visual}, which inserts static learnable parameters into every block of the vision encoder, our visual prompts are generated dynamically per image. \emph{Vision-to-Text (V2T) refinement}: a V2T adapter (in the spirit of CoCoOp~\cite{zhou2022conditional}) maps the resulting visual feature back into a per-instance shift on the text prompt, yielding a textual description that is tied to the specific image rather than shared across the dataset.

Iterating these two updates alone, however, risks error accumulation: noise in one branch can propagate to the other, and the iterates may drift away from the pre-trained manifold. We therefore introduce a \emph{Semantic Anchor}: the loop is initialised, and implicitly constrained, by the embedding of a generic class template (a \emph{Semantic Bootstrapping} step). Together with a multi-view aggregation module that suppresses per-augmentation sampling noise, the anchor combines the stability of the pre-trained prior with the flexibility of per-instance adaptation~\cite{liu2025self, mao2022understanding, sheng2025r, zanella2024test}.

Our main contributions are as follows.
\begin{itemize}
    \item We identify the \emph{independent treatment of modalities} as a common limitation of existing VLM defences, and propose CLBP—a bidirectional feedback loop in which visual and textual features mutually condition and correct each other at inference time, with frozen pre-trained encoders.

    \item We complement the loop with a \emph{Semantic Anchor} that initialises and constrains the iteration to remain near the pre-trained semantic manifold, and we analyse the resulting fixed-point dynamics: a single bidirectional update is provably and empirically already close to the loop's fixed point.

    \item Across 11 datasets and three evaluation regimes (cross-dataset zero-shot, few-shot adaptation, and base-to-new generalisation), CLBP improves robustness over representative AFT, APT, and test-time baselines while preserving the clean accuracy of the underlying CLIP backbone.
\end{itemize}

\begin{figure}[t]
\centering

\begin{subfigure}{0.32\textwidth}
  \centering
  \includegraphics[page=1, width=\linewidth, keepaspectratio]{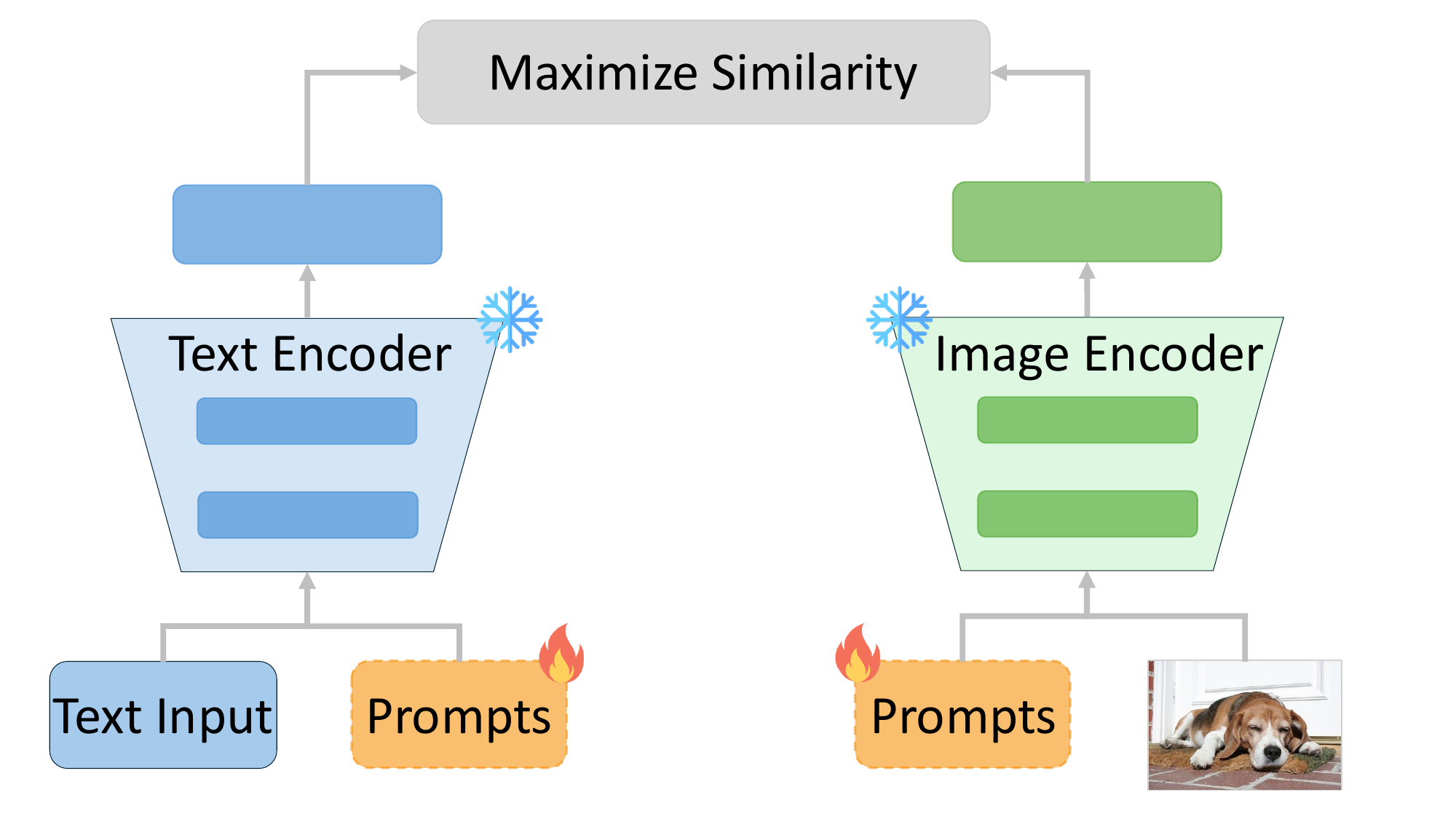}
  \caption{AdvVLP (Independent)}
  \label{fig:advvlp}
\end{subfigure}
\hfill 
\begin{subfigure}{0.32\textwidth}
  \centering
  \includegraphics[page=2, width=\linewidth, keepaspectratio]{figs/architecture.pdf}
  \caption{AdvMaPLe (Multi-modal)}
  \label{fig:maple}
\end{subfigure}
\hfill 
\begin{subfigure}{0.32\textwidth}
  \centering
  \includegraphics[page=3, width=\linewidth, keepaspectratio]{figs/architecture.pdf}
  \caption{CLBP (Cross-modal, Ours)}
  \label{fig:clbp}
\end{subfigure}
\vspace{1em}

\caption{\textbf{Comparison of prompt-learning paradigms.}
(a) \textbf{AdvVLP} inserts independent prompts in each modality with no cross-modal interaction (prone to overfitting).
(b) \textbf{AdvMaPLe} adds a unidirectional text-to-vision projection but remains static, lacking visual feedback for instance-specific adaptation.
(c) Our \textbf{CLBP} introduces a dynamic cross-modal feedback loop: lightweight networks align features across modalities to generate per-image defensive signals while keeping the pre-trained encoder frozen.}
\label{fig:comparison}
\vspace{-1em}
\end{figure}

\section{Related Work}
\label{sec:rw}

\textbf{Adversarial Vulnerability of Vision Language Models.}
Despite their impressive zero-shot capabilities, VLMs are sensitive to adversarial perturbations~\cite{mao2022understanding}. In contrast to attacks on uni-modal classifiers that aim to cross a decision boundary~\cite{goodfellow2014explaining}, attacks on VLMs typically introduce imperceptible noise that maximises the semantic distance between matched image--text pairs. Standard gradient-based attacks such as Projected Gradient Descent (PGD)~\cite{madry2017towards} and ensemble frameworks like AutoAttack~\cite{croce2020reliable} are routinely used to exploit white-box or gray-box access to the encoder, motivating defences that mitigate adversarial threats without compromising the model's generalisation.

\textbf{Training-Based Defence.}
Training-based defences for VLMs largely fall into two categories.
\textit{Adversarial Fine-Tuning} updates the visual encoder under an adversarial objective: TeCoA~\cite{mao2022understanding} aligns adversarial visual features with clean text embeddings, and FARE~\cite{schlarmann2024robust} minimises the distributional divergence between clean and adversarial samples. Both are effective on benchmarks such as ImageNet-1k but are computationally heavy and risk distorting the pre-trained feature space, which can hurt downstream performance.
\textit{Adversarial Prompt Tuning (APT)}~\cite{li2024one, zhou2022learning, zhang2024adversarial} sidesteps this cost by optimising a small set of soft prompts while keeping the backbone frozen. Early variants—AdvVP~\cite{mao2022understanding}, AdvPT~\cite{zhang2024adversarial} and AdvVLP~\cite{zhou2024few}—use independent prompts for each modality and largely ignore cross-modal correlations (Figure~\ref{fig:advvlp}). Subsequent works bridge the two branches via a unidirectional projection (Figure~\ref{fig:maple}): AdvMaPLe~\cite{zhou2024few}, built on MaPLe~\cite{khattak2023maple}, projects text prompts into the visual branch, while FAP inverts the direction. These methods, however, still produce static, instance-agnostic prompts and lack a mutual correction loop, which limits their adaptation to diverse adversarial patterns under few-shot supervision.

\textbf{Test-Time Defence.}
Test-time defences operate during inference and adapt per-instance without retraining. Test-Time Counterattack (TTC)~\cite{xing2025clip} exploits the anomalous stability of adversarial examples under secondary noise: it identifies the attack via noise injection and optimises a counter-perturbation that maximises feature deviation from the adversarial input. Robust Test-Time Prompt Tuning (R-TPT)~\cite{sheng2025r} attenuates perturbations through multi-view consistency and refines predictions via point-wise entropy minimisation. Self-calibrated Consistency (SCC)~\cite{liu2025self} extends this line by addressing semantic drift and view sensitivity through semantic and spatial consistency constraints. These methods all operate on the visual side and rely on consistency as the supervisory signal; the textual prototypes that determine cross-modal alignment~\cite{radford2021learning} remain unchanged.

\begin{figure}[t] 
  \centering
  \includegraphics[page=4, width=0.9\linewidth]{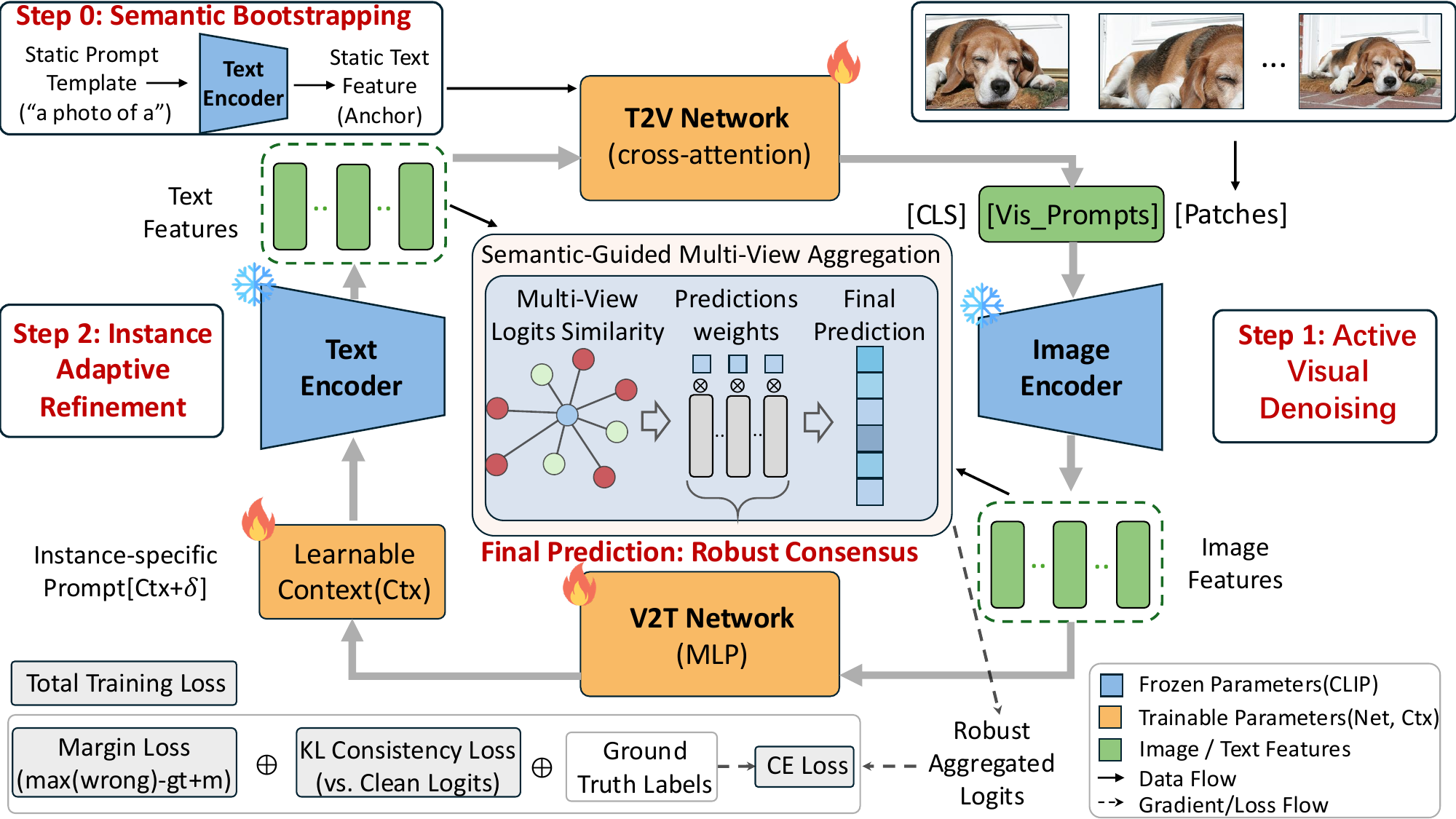}
  
\caption{\textbf{The CLBP pipeline.} \textbf{Step 0 (Semantic Bootstrapping)} initializes a fixed text anchor; \textbf{Step 1 (Active Visual Denoising, T2V)} filters image noise using the anchor; \textbf{Step 2 (Instance-Adaptive Refinement, V2T)} injects visual bias back into the text prompt, closing the loop. Final predictions are obtained by \textbf{Multi-View Aggregation} across augmented views; the model is trained with a tri-component loss.}
  \label{fig:main_workflow} 
  \vspace{-1em}
\end{figure}



\section{Method}
\label{sec:method}

\subsection{Preliminaries and Threat Model}
\label{sec:prelim}

We consider a pre-trained VLM such as CLIP, with image encoder $f_{\mathcal{V}}$ and text encoder $f_{\mathcal{T}}$.
Given an input image $\mathbf{x}$ and a set of class prompts $\{t_c\}_{c=1}^C$, we obtain unit-normalised embeddings $\mathbf{z}(\mathbf{x})\!=\!f_{\mathcal V}(\mathbf{x})/\|f_{\mathcal V}(\mathbf{x})\|_2$ and $\mathbf{w}_c\!=\!f_{\mathcal T}(t_c)/\|f_{\mathcal T}(t_c)\|_2$, and predict via a temperature-scaled cosine similarity:
\begin{equation}
  p(y\!=\!c\mid \mathbf{x}) \;=\;
  \frac{\exp(\langle \mathbf{z}(\mathbf{x}), \mathbf{w}_c\rangle/\tau)}
       {\sum_{j=1}^C \exp(\langle \mathbf{z}(\mathbf{x}), \mathbf{w}_j\rangle/\tau)},
  \label{eq:clip_softmax}
\end{equation}
with temperature $\tau\!>\!0$ and $C$ classes.

\paragraph{Threat model.}
We adopt a white-box adversary with access to all defence components (adapters, prompts, and the frozen-but-differentiable encoders), which crafts the worst-case perturbation within an $\ell_p$-budget against the full pipeline,
\begin{equation*}
  \mathbf{x}^{\mathrm{adv}} = \mathbf{x} + \arg\max_{\|\boldsymbol{\delta}\|_p\le\epsilon}\mathcal{L}\!\big(\mathrm{CLBP}(\mathbf{x}+\boldsymbol{\delta}),y\big).
\end{equation*}
This is the strongest evaluation regime for adversarial prompt tuning, and the one we adopt throughout.

\subsection{Closed-Loop Bidirectional Prompting}
\label{sec:clbp}

Adversarial perturbations on the input image shift the visual embedding away from its matching text prototype while leaving the textual prototypes unchanged.
AFT restores alignment indirectly by retraining the visual encoder, and APT does so by learning a static prompt; in both cases only one modality is adjusted.
CLBP recovers the alignment directly, by alternating textual denoising and visual refinement under a fixed semantic anchor.

We write $\mathrm{Norm}(\mathbf{u})\!=\!\mathbf{u}/\|\mathbf{u}\|_2$ and let $s\!=\!\exp(\mathrm{logit\_scale})$ denote the CLIP temperature scale.
The encoders $f_{\mathcal V},f_{\mathcal T}$ remain frozen; the trainable parameters are the T2V adapter $\Phi_{\mathrm{T2V}}$, the V2T adapter $\Theta_{\mathrm{V2T}}$, and the learnable context tokens of $\mathrm{Compose}(\cdot)$.
One closed-loop iteration alternates between the visual and textual states,
$\mathbf{W}^{(k)} \xrightarrow{\,\Phi_{\mathrm{T2V}}\,} \mathbf{z}^{(k)}(\mathbf{x}) \xrightarrow{\,\Theta_{\mathrm{V2T}}\,} \mathbf{W}^{(k+1)}(\mathbf{x})$,
and is initialised by a fixed semantic anchor $\mathbf{W}^{(0)}\!=\![\mathbf{w}_1^{(0)},\dots,\mathbf{w}_C^{(0)}]^{\!\top}$ with $\mathbf{w}_c^{(0)}\!=\!\mathrm{Norm}(f_{\mathcal T}(\mathrm{Template}(c)))$ obtained from a generic template (e.g., ``\texttt{a photo of a \{class\}}''). The anchor keeps the iterates close to the pre-trained text manifold throughout the loop.

\paragraph{Closed-loop update.}
Given the current text prototypes $\mathbf{W}^{(k)}$, the T2V step generates visual prompts from $\mathbf{W}^{(k)}$ and feeds them into the image encoder; the V2T step maps the resulting visual feature back into a per-instance text shift:
\begin{align}
  & \text{T2V:} & \mathbf{P}^{(k)} &= \Phi_{\mathrm{T2V}}\!\big(\mathbf{W}^{(k)}\big), &
  \mathbf{z}^{(k)}(\mathbf{x}) &= \mathrm{Norm}\!\big(f_{\mathcal V}(\mathbf{x}\mid \mathbf{P}^{(k)})\big),
  \label{eq:t2v_main} \\[2pt]
  & \text{V2T:} & \boldsymbol{\delta}^{(k+1)} &= \Theta_{\mathrm{V2T}}\!\big(\mathbf{z}^{(k)}(\mathbf{x})\big), &
  \mathbf{w}_c^{(k+1)} &= \mathrm{Norm}\!\Big(f_{\mathcal T}\big(\mathrm{Compose}(t_c^{(0)},\boldsymbol{\delta}^{(k+1)})\big)\Big).
  \label{eq:v2t_main}
\end{align}
$\mathrm{Compose}(\cdot)$ implements a lightweight context-token shift via a small multi-layer perceptron (MLP) that maps $\boldsymbol{\delta}^{(k+1)}$ to a per-token additive bias on the learnable context tokens, with the class-name suffix tokens kept fixed.

Iterating Eqs.~\eqref{eq:t2v_main}--\eqref{eq:v2t_main} for $k\!=\!0,\dots,N{-}1$ and stacking the refined prototypes into $\mathbf{W}^{(N)}(\mathbf{x})$, we read out the final logits as
\begin{equation}
  \boldsymbol{\ell}(\mathbf{x})
  \;=\;
  s\cdot \mathbf{z}^{(N)}(\mathbf{x})\,\big(\mathbf{W}^{(N)}(\mathbf{x})\big)^{\!\top} \in \mathbb{R}^C,
  \label{eq:logits}
\end{equation}
with probabilities given by Eq.~\eqref{eq:clip_softmax}; the full inference procedure is summarised in Algorithm~\ref{alg:clbp_inference} (Appendix~\ref{app:algo}).
This construction raises two questions: (i) since the loop iterates an adversarially-affected branch through another adversarially-affected branch, why should it \emph{converge} and stay close to the pre-trained semantics; and (ii) how many iterations $N$ are needed in practice. We address both before turning to the per-input variance handled by multi-view aggregation.

\paragraph{Stability of the loop.}
\label{sec:theory_main}
We view one closed-loop iteration as a fixed-point map on the unit sphere.
Let $\mathcal{G}\!:\!\mathbb{S}^{d-1}\!\to\!\mathbb{R}^{C\times d}$ denote the V2T map $\mathcal{G}(\mathbf{z})\!=\!\mathbf{W}(\mathbf{z})$ defined by Eq.~\eqref{eq:v2t_main}, and $\mathcal{H}_{\mathbf{x}}\!:\!\mathbb{R}^{C\times d}\!\to\!\mathbb{S}^{d-1}$ the T2V map $\mathcal{H}_{\mathbf{x}}(\mathbf{W})\!=\!\mathrm{Norm}(f_{\mathcal V}(\mathbf{x}\!\mid\!\Phi_{\mathrm{T2V}}(\mathbf{W})))$ defined by Eq.~\eqref{eq:t2v_main}.
The closed-loop update is then their composition
\begin{equation}
  \mathcal{T}_{\mathbf{x}} \;\triangleq\; \mathcal{H}_{\mathbf{x}}\!\circ\!\mathcal{G},
  \qquad
  \mathbf{z}^{(k+1)} \;=\; \mathcal{T}_{\mathbf{x}}\!\big(\mathbf{z}^{(k)}\big),
  \label{eq:Tx}
\end{equation}
restricted to the reachable set $\Omega\!\subseteq\!\mathbb{S}^{d-1}$ induced by the anchor initialisation $\mathbf{z}^{(0)}\!=\!\mathcal{H}_{\mathbf{x}}(\mathbf{W}^{(0)})$ (formally defined in Appendix~\ref{app:setup}).
Whether the loop accumulates or attenuates input perturbations is then governed by a single regularity quantity, the local Lipschitz constant $q\!\triangleq\!\mathrm{Lip}_\Omega(\mathcal{T}_{\mathbf{x}})$, and we are interested in the contractive regime $q\!<\!1$.
The factorised bound $q\!\le\!L_{\mathcal{G}}L_{\mathcal{H}_{\mathbf{x}}}$ is generally loose for CLBP because $\mathrm{Compose}(\cdot)$ projects the high-dimensional shift $\boldsymbol{\delta}^{(k+1)}$ through a low-rank context-token bottleneck: on held-out trajectories of the trained model we measure $L_{\mathcal{G}}\!\approx\!1.4$ and $L_{\mathcal{H}_{\mathbf{x}}}\!\approx\!1.0$, while the composed estimate is $q\!\sim\!10^{-4}$ on every clean and adversarial trajectory we evaluate (Appendix~\ref{app:empirical_lipschitz}, Tab.~\ref{tab:empirical_lipschitz}).
Two consequences follow.

\begin{theorem}[Convergence]
\label{thm:conv_main}
If $q\!<\!1$, $\mathcal{T}_{\mathbf{x}}$ is a contraction on $\Omega$ and admits a unique fixed point $\mathbf{z}^\star\!\in\!\Omega$, with $\|\mathbf{z}^{(k)}-\mathbf{z}^\star\|_2 \le q^k\|\mathbf{z}^{(0)}-\mathbf{z}^\star\|_2$.
\end{theorem}

\begin{theorem}[Stability under adversarial perturbations]
\label{thm:stab_main}
Assume in addition that $\mathcal{H}_{(\cdot)}$ is input-Lipschitz: $\|\mathcal{H}_{\mathbf{x}+\boldsymbol{\delta}}(\mathbf{W})-\mathcal{H}_{\mathbf{x}}(\mathbf{W})\|_2\le L_x\|\boldsymbol{\delta}\|$ for $\mathbf{W}\!\in\!\mathcal{G}(\Omega)$.
Let $\mathbf{z}^\star(\mathbf{x})$ and $\mathbf{z}^\star(\mathbf{x}+\boldsymbol{\delta})$ be the fixed points of $\mathcal{T}_{\mathbf{x}}$ and $\mathcal{T}_{\mathbf{x}+\boldsymbol{\delta}}$. Then
$\|\mathbf{z}^\star(\mathbf{x}+\boldsymbol{\delta})-\mathbf{z}^\star(\mathbf{x})\|_2 \le \tfrac{L_x}{1-q}\|\boldsymbol{\delta}\|$.
\end{theorem}

Theorem~\ref{thm:conv_main} follows from the Banach fixed-point theorem applied to the contraction $\mathcal{T}_{\mathbf{x}}$, and Theorem~\ref{thm:stab_main} from a standard fixed-point perturbation argument; full proofs are in Appendix~\ref{app:proofs}.
Together, the two statements imply that the deviation from the clean fixed point grows linearly in the input perturbation but does not depend on the iteration depth $N$, which is the property the loop needs to be safely iterated at inference time.
With the empirical $q\!\sim\!10^{-4}$, the first iterate satisfies $\|\mathbf{z}^{(1)}-\mathbf{z}^\star\|_2 \le q\,\|\mathbf{z}^{(0)}-\mathbf{z}^\star\|_2$, so a single update is already essentially at the fixed point.
We confirm this directly: along inference trajectories on DTD, $d_k\!=\!\|\mathbf{z}^{(k)}-\mathbf{z}^\star\|_2$ takes values $1.9{\times}10^{-6}$, $8.9{\times}10^{-7}$, $8.5{\times}10^{-7}$ at $k\!=\!0,1,2$, with classification accuracy plateauing at $38.5\%, 54.5\%, 54.5\%$ (Tab.~\ref{tab:fixed_point_proximity}).
The first update therefore captures essentially all of the bidirectional correction, and we set the default to $N\!=\!1$.
Two further lemmas used downstream---a feature-space margin implies prediction invariance, and multi-view outliers are suppressed in expectation---are deferred to Appendix~\ref{app:morelemmas}.

\paragraph{Multi-view aggregation.}
The closed-loop update reduces the systematic adversarial drift on a single forward pass, but each individual augmentation $\mathbf{x}^{(v)}$ still carries its own sampling noise.
We therefore aggregate predictions across $V$ random augmentations of $\mathbf{x}$.
Given per-view outputs $(\boldsymbol{\ell}_v,\mathbf{z}_v)\!=\!\mathrm{CLBP}(\mathbf{x}^{(v)})$, we compute pairwise cosine similarities $\mathbf{S}_{ij}\!=\!\langle\mathbf{z}_i,\mathbf{z}_j\rangle$, take the top-$K$ neighbour set $\mathcal{N}_i\!=\!\mathrm{TopK}(\mathbf{S}_{i,:})$, and aggregate logits with a similarity-weighted softmax:
\begin{equation}
  \mathrm{sc}_i = \tfrac{1}{K}\!\!\sum_{j\in\mathcal{N}_i}\!\! \mathbf{S}_{ij},
  \qquad
  \alpha_i = \frac{\exp(\mathrm{sc}_i/\tau_{\mathrm{mv}})}{\sum_{r}\exp(\mathrm{sc}_r/\tau_{\mathrm{mv}})},
  \qquad
  \boldsymbol{\ell}_{\mathrm{agg}} = \sum_{v} \alpha_v\,\boldsymbol{\ell}_v.
  \label{eq:mv_main}
\end{equation}
Views whose embeddings disagree with the consensus receive small weight; a formal outlier-suppression bound is given in Appendix~\ref{app:mv_outlier}.
We set $V{=}32$ in all main-text experiments to match the protocol of prior multi-view test-time defences (R-TPT and SCC also use $V{=}32$). The view-count ablation in Appendix~\ref{app:viewcount} shows that CLBP's robust accuracy saturates from $V{=}8$ onward---the spread across $V\!\in\!\{8,16,32\}$ on DTD under PGD-100 ($\epsilon{=}4/255$) is within $1.1\%$, on the order of test-time sampling noise across random seeds. Smaller $V$ (down to $V{=}8$ at $0.092$\,s/img) is therefore a viable lower-cost option when inference cost is the binding constraint, with no measurable robustness loss.

\paragraph{Training objective.}
Letting $\boldsymbol{\ell}^{\mathrm{adv}}_{\mathrm{agg}}$ and $\boldsymbol{\ell}^{\mathrm{clean}}_{\mathrm{agg}}$ denote the aggregated logits on adversarial and clean inputs, we minimise
\[
\mathcal{L} \;=\; \mathcal{L}_{\mathrm{ce}}(\boldsymbol{\ell}^{\mathrm{adv}}_{\mathrm{agg}},y) + \lambda_{\mathrm{mar}}\,\mathcal{L}_{\mathrm{mar}} + \lambda_{\mathrm{kl}}\,\mathcal{L}_{\mathrm{kl}},
\]
where $\mathcal{L}_{\mathrm{mar}}$ enforces a feature-space margin on the adversarial logits and $\mathcal{L}_{\mathrm{kl}}$ distils the clean prediction into the adversarial branch. Full definitions and hyper-parameter settings are in Appendix~\ref{app:train}.

\section{Experiments}
\label{sec:experiments}

\subsection{Experimental Setups}

\paragraph{Datasets.} Following standard benchmarks for CLIP adversarial robustness, we evaluate on a suite of 11 datasets covering diverse visual domains: ImageNet-1k~\cite{deng2009imagenet} and Caltech101~\cite{fei2006one} for generic object recognition; OxfordPets~\cite{parkhi2012cats}, Flowers102~\cite{nilsback2008automated}, Food101~\cite{bossard2014food}, StanfordCars~\cite{krause20133d}, and FGVC-Aircraft~\cite{maji2013fine} for fine-grained classification; SUN397~\cite{xiao2010sun} for scene recognition; UCF101~\cite{soomro2012ucf101} for action recognition; and DTD~\cite{cimpoi2014describing} together with EuroSAT~\cite{helber2019eurosat} for specialised domains (textures and satellite imagery).

\begin{table}[t!]
\centering

\caption{\textbf{Cross-dataset zero-shot robustness on 10 datasets.} Training-based baselines follow a unified protocol: ImageNet, 5 epochs, 2-step PGD at $\epsilon_{\mathrm{train}}{=}1/255$, $\alpha{=}1/255$. Evaluation: white-box 10-step PGD at $\epsilon_a{=}1/255$ in $\ell_\infty$. We report Top-1 clean (Acc.) and robust (Rob.) accuracy (\%); $\Delta$ is the gap between CLBP and vanilla CLIP.}
\label{table-main}
\resizebox{\textwidth}{!}{
\begin{tabular}{c|c|c|ccc|cccc|ccccc|c}
\hline
\multirow{2}{*}{\textbf{Dataset}} & \multirow{2}{*}{Metric} & \multirow{2}{*}{CLIP}
& \multicolumn{3}{c|}{\textbf{Test-time Defense}}
& \multicolumn{4}{c|}{\textbf{Adversarial Fine-Tuning}}
& \multicolumn{5}{c|}{\textbf{Adversarial Prompt Tuning}}
& \multirow{2}{*}{$\Delta$} \\
& &
& Anti-adv & TTC & SCC
& CLIP-FT & TeCoA & PMG-AFT & FARE
& AdvVP & AdvVLP & AdvMaPLe & FAP & CLBP(ours)
& \\
\noalign{\kern-\cmidrulewidth}
\cmidrule(lr){1-1}\cmidrule(lr){2-2}\cmidrule(lr){3-15}\cmidrule(lr){16-16}

\multirow{2}{*}{ImageNet}
& Rob. & 1.15
& 8.67 & 38.41 & 49.77
& 0.93 & 18.89 & 21.43 & 14.00 & 13.62 & 22.38 & 22.40 & 22.33 & {\cellcolor[rgb]{0.969,0.988,1}}\textbf{58.42}
& {\cellcolor[rgb]{0.992,0.973,0.973}}\textcolor[rgb]{0.502,0,0}{+57.27} \\
& Acc. & 59.69
& 54.27 & 49.39 & 56.03
& 54.24 & 34.89 & 36.12 & 48.79 & 47.93 & 52.99 & 53.04 & 55.15 & {\cellcolor[rgb]{0.969,0.988,1}}65.96
& {\cellcolor[rgb]{0.992,0.973,0.973}}+6.27 \\

\multirow{2}{*}{Caltech101}
& Rob. & 14.67
& 34.81 & 65.78 & 77.25
& 14.21 & 55.51 & 61.08 & 50.74 & 51.32 & 66.00 & 66.94 & 66.13 & {\cellcolor[rgb]{0.969,0.988,1}}\textbf{86.94}
& {\cellcolor[rgb]{0.992,0.973,0.973}}\textcolor[rgb]{0.502,0,0}{+72.27} \\
& Acc. & 85.66
& 84.02 & 86.53 & 86.44
& 83.63 & 71.68 & 75.45 & 80.95 & 86.09 & 86.69 & 87.63 & 87.38 & {\cellcolor[rgb]{0.969,0.988,1}}91.12
& {\cellcolor[rgb]{0.992,0.973,0.973}}+5.46 \\

\multirow{2}{*}{OxfordPets}
& Rob. & 1.04
& 20.42 & 57.87 & 76.67
& 2.10 & 38.35 & 41.18 & 31.07 & 22.95 & 42.19 & 43.58 & 43.85 & {\cellcolor[rgb]{0.969,0.988,1}}\textbf{81.68}
& {\cellcolor[rgb]{0.992,0.973,0.973}}\textcolor[rgb]{0.502,0,0}{+80.64} \\
& Acc. & 87.44
& 80.62 & 83.35 & 86.48
& 84.14 & 62.12 & 65.88 & 79.37 & 78.00 & 79.18 & 78.00 & 80.02 & {\cellcolor[rgb]{0.969,0.988,1}}85.04
& {\cellcolor[rgb]{0.992,0.973,0.973}}-2.4 \\

\multirow{2}{*}{Flowers102}
& Rob. & 1.14
& 7.16 & 39.14 & 54.59
& 0.54 & 21.94 & 23.43 & 17.14 & 20.42 & 25.01 & 28.58 & 26.84 & {\cellcolor[rgb]{0.969,0.988,1}}\textbf{55.42}
& {\cellcolor[rgb]{0.992,0.973,0.973}}\textcolor[rgb]{0.502,0,0}{+54.28} \\
& Acc. & 65.46
& 62.66 & 64.16 & 64.16
& 53.37 & 36.80 & 37.00 & 47.98 & 49.45 & 51.36 & 53.07 & 53.67 & {\cellcolor[rgb]{0.969,0.988,1}}63.58
& {\cellcolor[rgb]{0.992,0.973,0.973}}-1.88 \\

\multirow{2}{*}{FGVC-Aircraft}
& Rob. & 0.00
& 1.27 & 13.77 & 17.40
& 0.00 & 2.49 & 2.22 & 1.35 & 1.65 & 2.91 & 2.19 & 1.92 & {\cellcolor[rgb]{0.969,0.988,1}}\textbf{15.18}
& {\cellcolor[rgb]{0.992,0.973,0.973}}\textcolor[rgb]{0.502,0,0}{+15.18} \\
& Acc. & 20.10
& 15.88 & 18.00 & 17.61
& 14.04 & 5.31 & 5.55 & 10.86 & 9.45 & 11.31 & 9.93 & 10.71 & {\cellcolor[rgb]{0.969,0.988,1}}19.71
& {\cellcolor[rgb]{0.992,0.973,0.973}}-0.39 \\

\multirow{2}{*}{StanfordCars}
& Rob. & 0.02
& 4.40 & 33.01 & 43.24
& 0.06 & 8.76 & 11.65 & 6.75 & 2.98 & 11.38 & 11.25 & 10.57 & {\cellcolor[rgb]{0.969,0.988,1}}\textbf{48.97}
& {\cellcolor[rgb]{0.992,0.973,0.973}}\textcolor[rgb]{0.502,0,0}{+48.95} \\
& Acc. & 52.02
& 36.21 & 48.16 & 51.19
& 42.11 & 20.91 & 25.44 & 38.68 & 22.00 & 39.58 & 40.59 & 39.52 & {\cellcolor[rgb]{0.969,0.988,1}}60.05
& {\cellcolor[rgb]{0.992,0.973,0.973}}+8.03 \\

\multirow{2}{*}{SUN397}
& Rob. & 1.14
& 8.05 & 41.52 & 51.19
& 0.94 & 19.39 & 22.58 & 14.91 & 14.11 & 22.39 & 21.91 & 22.08 & {\cellcolor[rgb]{0.969,0.988,1}}\textbf{57.06}
& {\cellcolor[rgb]{0.992,0.973,0.973}}\textcolor[rgb]{0.502,0,0}{+55.92} \\
& Acc. & 58.50
& 56.00 & 55.13 & 53.27
& 55.73 & 36.69 & 37.98 & 52.42 & 42.51 & 53.92 & 53.17 & 54.44 & {\cellcolor[rgb]{0.969,0.988,1}}65.24
& {\cellcolor[rgb]{0.992,0.973,0.973}}+6.74 \\

\multirow{2}{*}{Food101}
& Rob. & 0.70
& 13.12 & 57.84 & 65.39
& 0.42 & 13.90 & 18.57 & 11.65 & 11.51 & 18.77 & 19.28 & 18.61 & {\cellcolor[rgb]{0.969,0.988,1}}\textbf{66.12}
& {\cellcolor[rgb]{0.992,0.973,0.973}}\textcolor[rgb]{0.502,0,0}{+65.42} \\
& Acc. & 83.88
& 75.81 & 82.18 & 82.13
& 64.86 & 29.98 & 36.61 & 55.31 & 58.76 & 55.64 & 56.76 & 58.78 & {\cellcolor[rgb]{0.969,0.988,1}}77.46
& {\cellcolor[rgb]{0.992,0.973,0.973}}-6.42 \\

\multirow{2}{*}{EuroSAT}
& Rob. & 0.03
& 2.15 & 12.19 & 20.64
& 0.04 & 11.96 & 12.60 & 10.67 & 2.37 & 11.01 & 11.09 & 10.79 & {\cellcolor[rgb]{0.969,0.988,1}}\textbf{18.96}
& {\cellcolor[rgb]{0.992,0.973,0.973}}\textcolor[rgb]{0.502,0,0}{+18.93} \\
& Acc. & 42.59
& 36.78 & 53.24 & 41.69
& 27.64 & 16.58 & 18.53 & 21.88 & 20.17 & 18.25 & 20.49 & 18.80 & {\cellcolor[rgb]{0.969,0.988,1}}26.86
& {\cellcolor[rgb]{0.992,0.973,0.973}}-15.73 \\

\multirow{2}{*}{DTD}
& Rob. & 2.98
& 5.62 & 27.32 & 34.57
& 2.39 & 17.61 & 14.95 & 15.64 & 15.31 & 19.50 & 16.49 & 16.49 & {\cellcolor[rgb]{0.969,0.988,1}}\textbf{38.95}
& {\cellcolor[rgb]{0.992,0.973,0.973}}\textcolor[rgb]{0.502,0,0}{+35.97} \\
& Acc. & 40.64
& 38.92 & 36.98 & 37.34
& 36.49 & 25.16 & 21.76 & 32.07 & 32.74 & 32.57 & 28.31 & 30.56 & {\cellcolor[rgb]{0.969,0.988,1}}44.86
& {\cellcolor[rgb]{0.992,0.973,0.973}}+4.22 \\

\hline\hline
\multirow{2}{*}{Avg.}
& Rob. & 2.29
& 10.57 & 38.69 & 49.71
& 2.16 & 20.88 & 22.97 & 17.39 & 15.62 & 24.15 & 24.37 & 23.96 & {\cellcolor[rgb]{0.969,0.988,1}}\textbf{52.77}
& {\cellcolor[rgb]{0.992,0.973,0.973}}\textcolor[rgb]{0.502,0,0}{+50.48} \\
& Acc. & 59.60
& 54.12 & 57.71 & 57.63
& 51.63 & 34.01 & 36.03 & 48.55 & 44.71 & 48.15 & 48.10 & 48.90 & {\cellcolor[rgb]{0.969,0.988,1}}59.99
& {\cellcolor[rgb]{0.992,0.973,0.973}}+0.39 \\
\hline
\end{tabular}}
\end{table}

\paragraph{Baselines.} We compare against representative methods spanning the three defence paradigms.
For \emph{test-time defences} we evaluate Anti-Adv (image--text consistency maximisation)~\cite{alfarra2022combating}, TTC (counter-perturbation optimisation)~\cite{xing2025clip}, and R-TPT (multi-view entropy minimisation)~\cite{sheng2025r}, alongside standard adaptation methods (TPT~\cite{shu2022test}, C-TPT~\cite{yoon2024c}, MTA~\cite{zanella2024test}).
For \emph{training-based defences} we adopt a strict \textit{cross-dataset transfer protocol} to assess generalisation: models are trained on a large-scale source dataset (ImageNet) and evaluated directly on unseen downstream domains. Within this protocol we compare two sub-categories:
(i) \emph{Adversarial Fine-Tuning} (AFT): TeCoA (semantic guidance)~\cite{mao2022understanding}, FARE (unsupervised)~\cite{schlarmann2024robust}, PMG-AFT~\cite{wang2024pre}, and a clean fine-tuned baseline (CLIP-FT); and
(ii) \emph{Adversarial Prompt Tuning} (APT) in the few-shot regime: AdvVP (visual prompts only)~\cite{mao2022understanding}, AdvVLP (independent dual prompts)~\cite{zhou2024few}, AdvMaPLe (unidirectional deep prompting)~\cite{zhou2024few}, and FAP (interaction-optimised prompting)~\cite{zhou2024few}.

\paragraph{Implementation Details.} We use CLIP with Vision Transformer (ViT) backbones (ViT-B/16 and ViT-B/32) as the visual encoders. The semantic anchor is initialised from the standard manual template \texttt{"a photo of a [class]"} on every dataset, and both the visual and textual learnable prompts use a context length of 4 tokens (sensitivity in Appendix~\ref{app:prompt_length}). Following the schedule in Appendix~\ref{app:optimization}, models are trained for 5 epochs under the cross-dataset transfer benchmark and 10 epochs in the few-shot setting; adversarial perturbations are generated by a 2-step PGD attack at $\epsilon{=}1/255$ and $\alpha{=}1/255$ in $\ell_\infty$. We report Top-1 accuracy on both clean and adversarial examples. At evaluation we use a suite of white-box and adaptive attacks: PGD-$\ell_\infty$~\cite{madry2017towards}, the ensemble framework AutoAttack~\cite{croce2020reliable}, the margin-based Carlini--Wagner attack (CW-$\ell_\infty$)~\cite{carlini2017towards}, and the adaptive Expectation-over-Transformation PGD (EOT-PGD)~\cite{athalye2018obfuscated}; their roles, budgets and results are presented alongside the corresponding analyses in Section~\ref{sec:experiments}. The defence mechanism is visualised in Appendix~\ref{sec:appendix_analysis}, and all experiments are run on NVIDIA H100 GPUs.

\begin{table}[t!]
\centering
\caption{\textbf{16-shot adversarial adaptation across 11 datasets.} Training: 10 epochs, 2-step PGD at $\epsilon_{\mathrm{train}}{=}1/255$, $\alpha{=}1/255$. Evaluation: 100-step PGD at $\epsilon_a{=}1/255$ in $\ell_\infty$. Top-1 clean (Acc.) and robust (Rob.) accuracy (\%); best in each cell is \textbf{bold}.}
\label{table-few-shot-clean}
\resizebox{\textwidth}{!}{
\begin{tabular}{l|c| ccccccccccc|c}
\toprule
\textbf{Method} & \textbf{Metric} & 
\textbf{ImageNet} & \textbf{Caltech101} & \textbf{OxfordPets} & \textbf{Flowers102} & \textbf{Aircraft} & \textbf{Cars} & \textbf{SUN397} & \textbf{Food101} & \textbf{EuroSAT} & \textbf{DTD} & \textbf{UCF101} & \textbf{Avg.} \\
\midrule

\multirow{2}{*}{AdvVP} 
& Rob. & 12.77 & 52.60 & 16.43 & 22.03 & 0.63 & 3.57 & 17.63 & 0.80 & 15.83 & 13.87 & 0.93 & 14.28 \\
& Acc. & 46.27 & 90.40 & 56.40 & 56.17 & 1.33 & 14.83 & 54.70 & 1.07 & 18.13 & 29.20 & 0.97 & 33.59 \\

\multirow{2}{*}{AdvVLP} 
& Rob. & 22.10 & 67.97 & 35.57 & 58.70 & 8.40 & 17.47 & 29.70 & 16.50 & 17.30 & 32.72 & 32.80 & 30.84 \\
& Acc. & 53.23 & 92.37 & 82.93 & 87.70 & 23.27 & 56.00 & 63.90 & 43.30 & 15.50 & 57.53 & 69.10 & 58.62 \\

\multirow{2}{*}{AdvMaPLe} 
& Rob. & 21.90 & 68.63 & 36.87 & 58.70 & 7.33 & 17.57 & 29.70 & 25.27 & 32.97 & 32.17 & 31.67 & 32.98 \\
& Acc. & 52.93 & 92.17 & 83.27 & 87.87 & 23.63 & 56.17 & 63.57 & 65.13 & 54.97 & 57.93 & 68.97 & 64.24 \\

\multirow{2}{*}{FAP} 
& Rob. & 22.90 & 67.33 & 41.00 & 61.47 & 7.97 & 19.23 & 30.27 & 26.67 & 39.73 & 31.33 & 32.80 & 34.61 \\
& Acc. & 52.53 & 91.10 & 81.90 & 86.27 & 23.50 & 54.23 & 62.37 & 64.03 & 81.70 & 55.17 & 65.70 & 65.32 \\

\midrule
\multirow{2}{*}{\textbf{CLBP}} 
& \textbf{Rob.} & {\cellcolor[rgb]{0.969,0.988,1}}\textbf{59.88} & {\cellcolor[rgb]{0.969,0.988,1}}\textbf{92.01} & {\cellcolor[rgb]{0.969,0.988,1}}\textbf{85.20} & {\cellcolor[rgb]{0.969,0.988,1}}\textbf{75.48} & {\cellcolor[rgb]{0.969,0.988,1}}\textbf{15.42} & {\cellcolor[rgb]{0.969,0.988,1}}\textbf{56.24} & {\cellcolor[rgb]{0.969,0.988,1}}\textbf{55.53} & {\cellcolor[rgb]{0.969,0.988,1}}\textbf{68.94} & {\cellcolor[rgb]{0.969,0.988,1}}\textbf{78.57} & {\cellcolor[rgb]{0.969,0.988,1}}\textbf{57.57} & {\cellcolor[rgb]{0.969,0.988,1}}\textbf{71.03} & {\cellcolor[rgb]{0.969,0.988,1}}\textbf{65.08} \\
& \textbf{Acc.} & {\cellcolor[rgb]{0.969,0.988,1}}\textbf{68.42} & {\cellcolor[rgb]{0.969,0.988,1}}\textbf{95.70} & {\cellcolor[rgb]{0.969,0.988,1}}\textbf{90.84} & {\cellcolor[rgb]{0.969,0.988,1}}\textbf{84.13} & {\cellcolor[rgb]{0.969,0.988,1}}\textbf{17.46} & {\cellcolor[rgb]{0.969,0.988,1}}\textbf{68.91} & {\cellcolor[rgb]{0.969,0.988,1}}\textbf{65.61} & {\cellcolor[rgb]{0.969,0.988,1}}\textbf{80.70} & {\cellcolor[rgb]{0.969,0.988,1}}\textbf{88.06} & {\cellcolor[rgb]{0.969,0.988,1}}\textbf{62.88} & {\cellcolor[rgb]{0.969,0.988,1}}\textbf{76.69} & {\cellcolor[rgb]{0.969,0.988,1}}\textbf{72.67} \\

\bottomrule
\end{tabular}}
\end{table}

\begin{figure}[t]
  \centering
  \includegraphics[width=\linewidth]{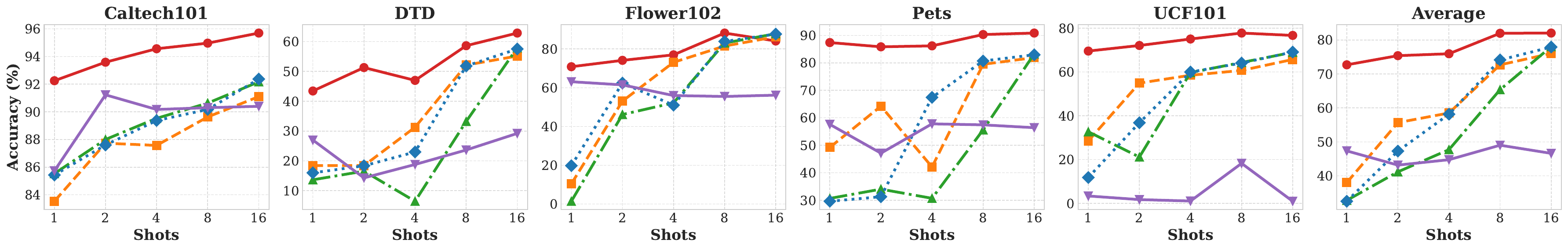}\\[2pt]
  \includegraphics[width=\linewidth]{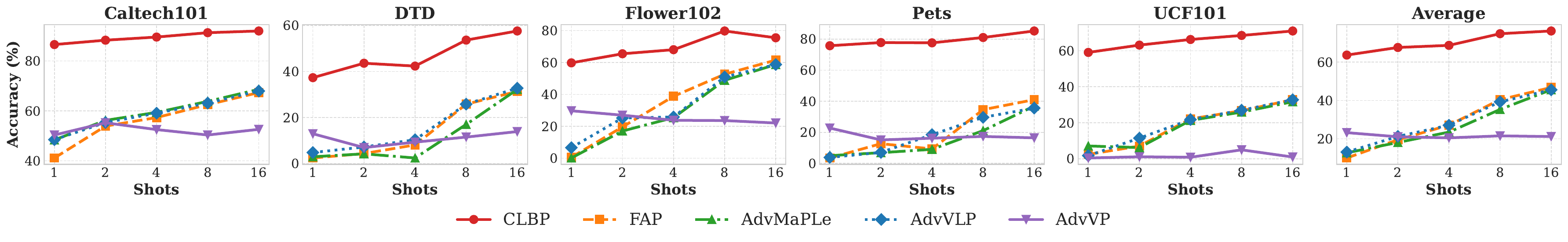}
  \caption{\textbf{Few-shot accuracy averaged over 11 datasets} for $n{\in}\{1,2,4,8,16\}$ shots/class. Training: 10 epochs, 2-step PGD at $\epsilon_{\mathrm{train}}{=}1/255$. \emph{Top}: Top-1 clean; \emph{Bottom}: Top-1 PGD-100 robust accuracy at $\epsilon_a{=}1/255$.}
  \label{fig:avg_combined}
\end{figure}

\subsection{Main Results}

\paragraph{Cross-dataset Robustness Evaluation.}
Under the cross-dataset zero-shot protocol (10 datasets, 10-step PGD at $\epsilon{=}1/255$, $\alpha{=}1/255$ in $\ell_\infty$), Table~\ref{table-main} reports CLBP at $52.77\%$ average robust accuracy: $3.06\%$ above the strongest test-time defence SCC ($49.71\%$), and roughly twice the strongest training-based baselines (AdvMaPLe $24.37\%$, FAP $23.96\%$, PMG-AFT $22.97\%$).
Clean accuracy is $59.99\%$, $0.39\%$ above vanilla CLIP, so the robustness gain is not paid for by a clean-accuracy drop.

\paragraph{Adversarial Few-shot Learning.}
We evaluate few-shot adaptation with $n\!\in\!\{1, 2, 4, 8, 16\}$ shots per class. Models are trained with 2-step PGD at $\epsilon{=}1/255$ and evaluated under PGD-100 at the same budget.

\textbf{16-shot.}
On all 11 datasets in Table~\ref{table-few-shot-clean}, CLBP attains $65.08\%$ average robust accuracy, while the closest baselines---FAP ($34.61\%$) and AdvMaPLe ($32.98\%$)---reach roughly half of that. We attribute the gap to the static prompts used by these baselines, which cannot adapt to per-input perturbations, whereas the bidirectional update in CLBP refines the prediction for each instance.

\textbf{Low-shot regime.}
Figure~\ref{fig:avg_combined} plots robustness as a function of the number of shots. At $n{=}1$, the average baseline robustness drops to $\sim\!2.43\%$, indicating that the static prompts overfit to the support perturbations; CLBP retains $41.13\%$ at the same shot count. The semantic anchor, which acts as a fixed prior on the per-instance adaptation, is the main source of this gap.

\paragraph{Adversarial Base-to-New Generalisation.}
We measure how well adversarial robustness transfers from seen (Base) to unseen (New) categories.
In Table~\ref{tab:avg_base2new}, FAP drops from $38.05\%$ on Base to $21.86\%$ on New, while CLBP retains $54.92\%$ on New, $33.06\%$ above the next best.
The drop on FAP is consistent with the static prompt overfitting to the Base classes; in CLBP the anchor-guided update keeps the prediction tied to the pre-trained text manifold and transfers across class boundaries. Per-dataset numbers are given in Appendix~\ref{basetonew_detail}.

\paragraph{Scalability and Versatility.}
We extend the evaluation along two axes: backbone scale, and attack diversity beyond the PGD-100 attacks used above.

\textbf{ViT-B/16 backbone.}
On the larger ViT-B/16 backbone with $\epsilon{=}4.0/255$, CLBP attains $48.2\%$ zero-shot robust accuracy, $+8.3\%$ over R-TPT (Appendix~\ref{sec:vit16}); the gain transfers to the finer-grained encoder.

\textbf{AutoAttack ensemble.}
On DTD with $\epsilon{=}4.0/255$, CLBP retains $31.86\%$ versus FAP's $10.47\%$ (Appendix~\ref{sec:autoattack}); robustness is preserved under a parameter-free Auto-PGD ensemble combining cross-entropy and Difference-of-Logits-Ratio losses (APGD-CE / APGD-DLR), ruling out hand-tuned PGD as the source of the gain.

\textbf{Carlini--Wagner ($\ell_\infty$) margin attack.}
The CW attack~\cite{carlini2017towards} replaces cross-entropy with a margin objective that drives the logit gap between the correct class and the strongest competitor below zero, producing tighter misclassifications near the decision boundary than untargeted PGD. At $\epsilon{=}1/255$, CLBP attains the highest robust accuracy on each of the 8 evaluated datasets, with an average $+10.0\%$ gap over the strongest baseline SCC (Appendix~\ref{app:cw}); the gain extends from gradient-magnitude to margin-based attacks.

\textbf{Adaptive EOT-PGD.}
Because the multi-view aggregation in Eq.~\eqref{eq:mv_main} introduces stochasticity in the forward pass, an attacker may underestimate the true gradient by sampling a single random draw---a failure mode known as \emph{gradient obfuscation}~\cite{athalye2018obfuscated}. EOT-PGD addresses this by drawing $\mathrm{eot}$ independent augmentation samples at each PGD step and averaging their per-sample gradients, giving an unbiased Monte-Carlo estimator of the expected loss gradient under the augmentation distribution; larger $\mathrm{eot}$ tightens the estimate. At $\mathrm{eot}{=}4$ on the unaltered stochastic pipeline, CLBP retains $73.14\%$ on Caltech101 and $28.72\%$ on DTD, $+27.34\%$ and $+15.95\%$ above R-TPT respectively (Appendix~\ref{app:eot}); the robustness gain is not an artefact of gradient obfuscation.

\textbf{Attack-strength sensitivity.}
Fig.~\ref{fig:ablations}(a) shows how CLBP degrades as the attacker's compute budget grows: as the number of PGD attack episodes (independent restarts) varies over $\{1,2,4,8\}$, robustness decreases monotonically from $85.2\%$ to $54.9\%$ on OxfordPets and from $57.6\%$ to $37.8\%$ on DTD, while remaining above the corresponding 1-episode baselines in Table~\ref{table-few-shot-clean} at every budget.

\subsection{Ablation Studies}
\label{sec:ablation}

We ablate the design choices of CLBP on Caltech101 and DTD: (i) which modules are responsible for the gain (Tab.~\ref{tab:ablation_study}), and (ii) how the multi-view budget trades off against latency (Fig.~\ref{fig:ablations}(c)). Robustness under the adaptive EOT-PGD attack introduced in Section~\ref{sec:experiments} is reported in Appendix~\ref{app:eot}.

\begin{table}[t]
\centering
\begin{minipage}[t]{0.42\linewidth}
  \centering
  \captionof{table}{\textbf{Adversarial base-to-new generalization.} Trained on Base classes (16-shot, 10 epochs, 2-step PGD at $\epsilon_{\mathrm{train}}{=}1/255$); evaluated on Base/New splits under PGD-100 at $\epsilon_a{=}1/255$. Reported: average Top-1 clean (Acc) / robust (Rob) accuracy (\%) across 11 datasets.}
  \label{tab:avg_base2new}
  \vspace{4pt}
  \small
  \setlength{\tabcolsep}{4pt}
  \renewcommand{\arraystretch}{1.2}
  \begin{tabular}{l cc cc}
    \toprule
    \multirow{2}{*}{\textbf{Method}}
      & \multicolumn{2}{c}{\textbf{Base}}
      & \multicolumn{2}{c}{\textbf{New}} \\
    \cmidrule(lr){2-3}\cmidrule(lr){4-5}
      & Acc & Rob & Acc & Rob \\
    \midrule
    AdvVP    & 31.68 & 14.43 & 30.39 & 13.36 \\
    AdvVLP   & 60.38 & 30.69 & 46.18 & 20.25 \\
    AdvMaPLe & 58.95 & 32.37 & 46.92 & 21.61 \\
    FAP      & 70.52 & 38.05 & 49.58 & 21.86 \\
    \midrule
    \textbf{CLBP (ours)}
      & \textbf{78.77} & \textbf{72.38}
      & \textbf{62.51} & \textbf{54.92} \\
    \bottomrule
  \end{tabular}
\end{minipage}\hfill
\begin{minipage}[t]{0.56\linewidth}
  \centering
  \captionof{table}{\textbf{Component analysis on Caltech101 and DTD.} Sem.\,=\,Semantic Anchor; T2V/V2T are the cross-modal adapters (Section~\ref{sec:clbp}); Loop\,=\,closed-loop alternation. Row~5 is the sequential T2V$\,{\to}\,$V2T baseline (no loop). 16-shot training (10 epochs, 2-step PGD at $\epsilon_{\mathrm{train}}{=}1/255$); PGD-100 evaluation at $\epsilon_a{=}1/255$. Top-1 clean (Acc.) / robust (Rob.) accuracy (\%).}
  \label{tab:ablation_study}
  \vspace{4pt}
  \small
  \setlength{\tabcolsep}{4pt}
  \renewcommand{\arraystretch}{1.05}
  \begin{tabular}{cccc | cc | cc}
    \toprule
    \multirow{2}{*}{Sem.} & \multirow{2}{*}{T2V} & \multirow{2}{*}{V2T} & \multirow{2}{*}{Loop} & \multicolumn{2}{c|}{Caltech101} & \multicolumn{2}{c}{DTD} \\
     & & & & Acc. & Rob. & Acc. & Rob. \\
    \midrule
    \xmark & \xmark & \xmark & --     & 90.99 & 84.46 & 42.44 & 35.58 \\
    \midrule
    \xmark & \cmark & \xmark & --     & 92.66 & 88.64 & 46.04 & 39.78 \\
    \xmark & \xmark & \cmark & --     & 94.24 & 90.43 & 54.61 & 50.18 \\
    \xmark & \cmark & \cmark & --     & 81.22 & 75.42 & 60.46 & 55.73 \\
    \midrule
    \cmark & \cmark & \cmark & \xmark & 94.52         & 91.16         & 61.46         & 56.68 \\
    \cmark & \cmark & \cmark & \cmark & \textbf{95.70}& \textbf{92.01}& \textbf{62.88}& \textbf{57.57} \\
    \bottomrule
  \end{tabular}
\end{minipage}
\end{table}

\begin{figure}[t]
\centering
\begin{subfigure}[t]{0.32\linewidth}
  \centering
  \includegraphics[width=\linewidth, height=3.8cm, keepaspectratio]{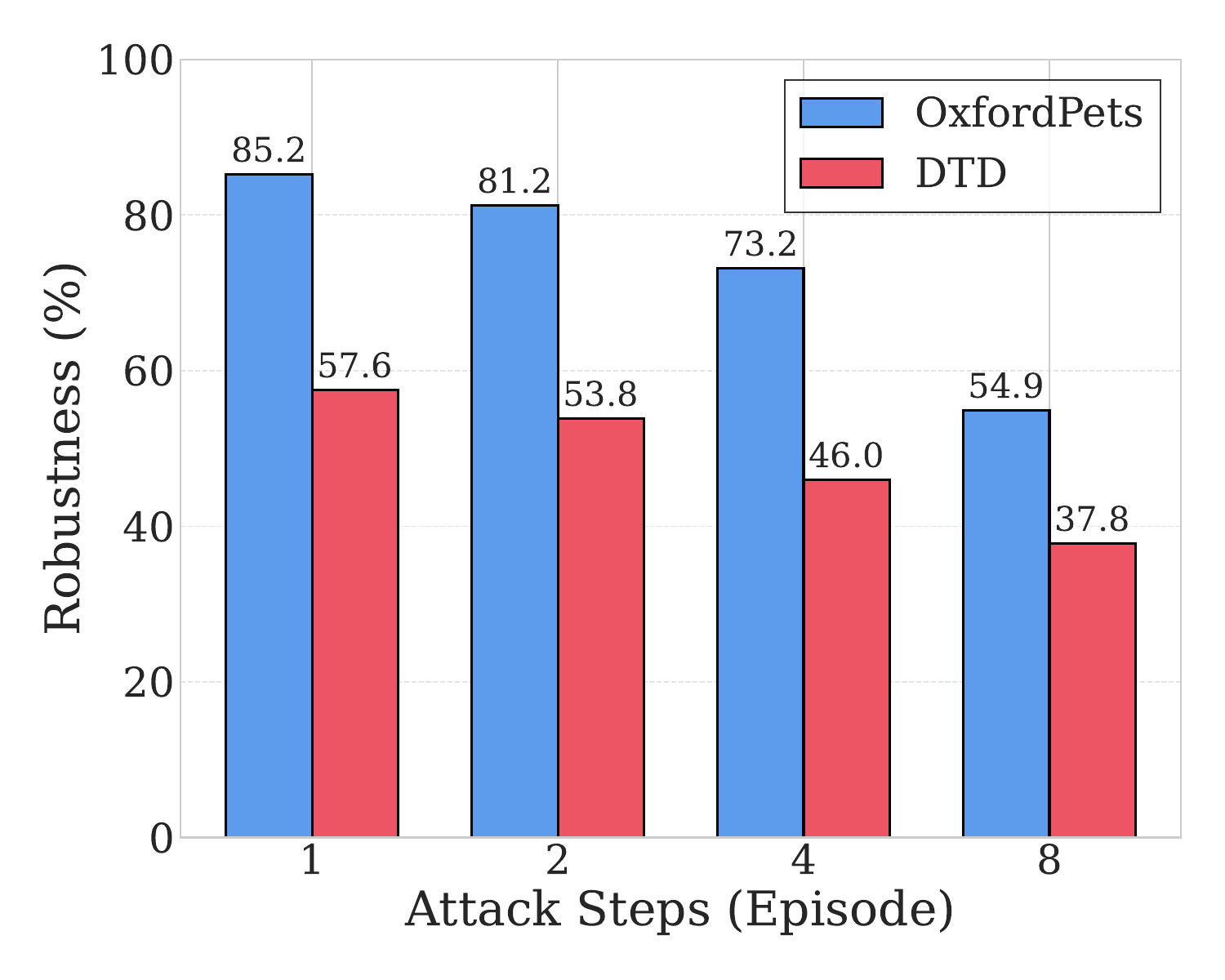}
  \caption{Attack-strength sensitivity.}
  \label{fig:episode-ablation}
\end{subfigure}\hfill
\begin{subfigure}[t]{0.32\linewidth}
  \centering
  \includegraphics[width=\linewidth, height=3.8cm, keepaspectratio]{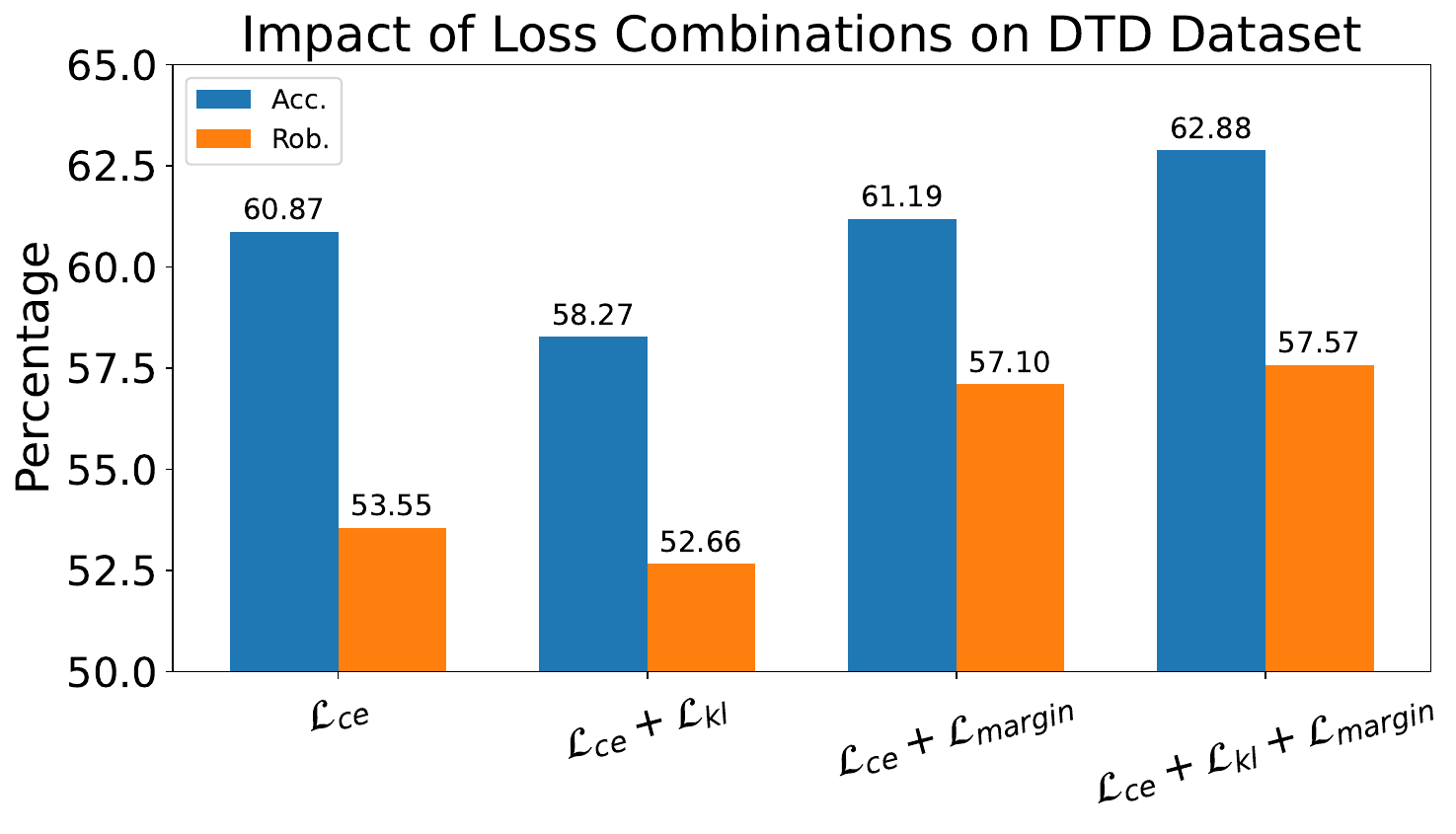}
  \caption{Loss-component ablation.}
  \label{fig:dtd-loss-ablation}
\end{subfigure}\hfill
\begin{subfigure}[t]{0.32\linewidth}
  \centering
  \includegraphics[width=\linewidth, height=3.8cm, keepaspectratio]{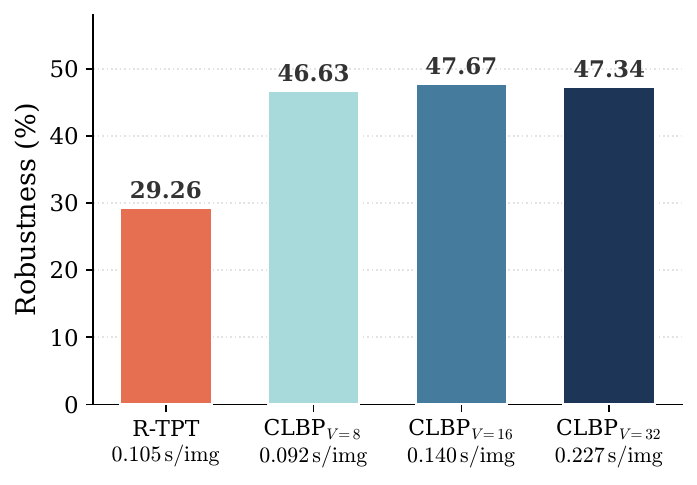}
  \caption{Efficiency vs.\ robustness.}
  \label{fig:efficiency-ablation}
\end{subfigure}
\caption{\textbf{Ablation on attack strength, training losses and view budget.} All panels use the 16-shot CLBP model (10 epochs, 2-step PGD at $\epsilon_{\mathrm{train}}{=}1/255$).
(a) Top-1 PGD-100 robust accuracy at $\epsilon_a{=}1/255$ on OxfordPets and DTD as the number of PGD attack episodes (independent restarts) varies in $\{1,2,4,8\}$.
(b) Impact of loss components $\mathcal{L}_{ce}, \mathcal{L}_{kl}, \mathcal{L}_{mar}$ (Section~\ref{sec:clbp}) on Top-1 clean / PGD-100 robust accuracy ($\epsilon_a{=}1/255$) on DTD.
(c) Per-image latency vs.\ Top-1 PGD-100 robust accuracy on DTD at the stronger $\epsilon_a{=}4/255$ for R-TPT and CLBP at $V{\in}\{8,16,32\}$.}
\label{fig:ablations}
\end{figure}

\paragraph{Module Analysis.}
Table~\ref{tab:ablation_study} ablates the Semantic Anchor, T2V denoising, and V2T refinement.
Either T2V or V2T alone improves over the baseline, and V2T contributes more on DTD than T2V.
Combining T2V and V2T \emph{without} the Semantic Anchor degrades clean accuracy on Caltech101 ($90.99\!\to\!81.22\%$), consistent with Theorem~\ref{thm:conv_main}: without the anchor, the loop has no fixed initialisation and the iterates can drift outside $\Omega$.
Adding the anchor recovers the loss and yields the best result on both datasets.

\paragraph{Efficiency vs.\ Robustness.}
Fig.~\ref{fig:ablations}(c) reports CLBP at $V\!\in\!\{8,16,32\}$ alongside R-TPT.
Robust accuracy is essentially flat across the three view budgets ($46.63\%$, $47.67\%$, $47.34\%$, within $1.1\%$), while latency grows roughly linearly from $0.092$ to $0.227$\,s/img.
This indicates that the multi-view consensus is already saturated at $V{=}8$, so the main-table results reported at $V{=}32$ are not driven by the larger view budget. We retain $V{=}32$ in the main experiments for parity with R-TPT and SCC, while $V{=}8$ remains a viable lower-cost setting that already exceeds R-TPT in both latency ($0.092$ vs.\ $0.105$\,s/img) and robustness ($46.63\%$ vs.\ $29.26\%$). The dependence on the closed-loop iteration depth $T$ is studied separately in Appendix~\ref{sec:efficiency}.

\paragraph{Training Loss Components.}
Fig.~\ref{fig:ablations}(b) ablates the three terms of the training objective $\mathcal{L} = \mathcal{L}_{\mathrm{ce}} + \lambda_{\mathrm{mar}}\mathcal{L}_{\mathrm{mar}} + \lambda_{\mathrm{kl}}\mathcal{L}_{\mathrm{kl}}$ (Section~\ref{sec:clbp}) on DTD.
$\mathcal{L}_{\mathrm{ce}}$ alone reaches $60.87\%$ clean and $53.55\%$ robust accuracy.
Adding $\mathcal{L}_{\mathrm{kl}}$ alone has a marginal effect on robustness, while adding $\mathcal{L}_{\mathrm{mar}}$ alone improves robustness by $+3.55\%$ ($53.55\!\to\!57.10\%$).
The full combination is best, reaching $62.88\%$ clean and $57.57\%$ robust accuracy: $\mathcal{L}_{\mathrm{mar}}$ supplies the bulk of the robustness gain, while $\mathcal{L}_{\mathrm{kl}}$ recovers the clean accuracy that $\mathcal{L}_{\mathrm{mar}}$ alone slightly sacrifices.

\section{Conclusion}
\label{sec:conclusion}

We presented Closed-Loop Bidirectional Prompting (CLBP), an adversarial defence for VLMs that recovers cross-modal alignment by alternating textual denoising and visual refinement under a fixed semantic anchor.
The composed loop is locally contractive on the reachable set induced by the anchor, with empirical Lipschitz constant $q\!\sim\!10^{-4}$, so a single bidirectional update is already close to the loop's fixed point.
Across 11 classification datasets, CLBP improves few-shot robustness, base-to-new generalisation, and cross-dataset transfer over the AFT, APT, and test-time defence baselines we evaluate.

\section{Limitations}
\label{sec:limitations}

We outline two aspects of CLBP whose extension we see as natural directions for future work.

\textbf{Architectural scope.}
Our analysis and empirical study focus on CLIP-style dual-tower VLMs, where the visual and textual branches are processed by separate encoders and combined through a cosine similarity. The bidirectional formulation is conceptually compatible with deeper-fusion architectures (e.g., BLIP-style models with intra-encoder cross-attention layers), but extending the closed-loop construction to such models may call for a different parametrisation of the T2V and V2T maps; we leave this extension to future work.

\textbf{Analytical tightness.}
The convergence result in Theorem~\ref{thm:conv_main} relies on a factorised Lipschitz bound $q\!\le\!L_{\mathcal{G}}L_{\mathcal{H}_{\mathbf{x}}}$ that is convenient but not tight in our setting: the empirical Lipschitz constant we measure ($q\!\sim\!10^{-4}$) is several orders of magnitude below the product of the per-branch constants. A tighter bound that better captures the contractive behaviour we observe, together with extensions to threat models beyond $\ell_\infty$ digital attacks—such as patch-based or physical-world threats—would further strengthen the theoretical picture.

\bibliographystyle{plainnat}
\bibliography{reference}

\newpage
\appendix


\section{Inference Pipeline}
\label{app:algo}

\subsection{Closed-Loop Inference with Multi-View Aggregation}
\label{app:algo_infer}

\begin{algorithm}[H]
\caption{CLBP Inference with Closed Loop and Multi-View Aggregation}
\label{alg:clbp_inference}
\begin{algorithmic}[1]
\REQUIRE \textbf{Data:} Image $\mathbf{x}$; Class set $\mathcal{C}$.
\REQUIRE \textbf{Modules:} Frozen encoders $f_{\mathcal V}, f_{\mathcal T}$; Adapters $\Phi_{\mathrm{T2V}}, \Theta_{\mathrm{V2T}}$.
\REQUIRE \textbf{params:} View generator $\mathcal{A}(\cdot)$ ($V$ views); Refinement depth $N$; Aggregation params $k, \tau_{\mathrm{mv}}$.
\ENSURE Prediction $\hat{y}$.

\vspace{0.3em}
\STATE $\{\mathbf{x}^{(v)}\}_{v=1}^V \leftarrow \mathcal{A}(\mathbf{x})$ \COMMENT{Generate augmented views}

\vspace{0.3em}
\STATE \textcolor{blue}{\textbf{// Step 0: Semantic Bootstrapping (Anchor)}}
\FORALL{$c \in \mathcal{C}$}
    \STATE $t_c^{(0)} \leftarrow \mathrm{Template}(c)$
    \STATE $\mathbf{w}_c^{(0)} \leftarrow \mathrm{Normalize}(f_{\mathcal T}(t_c^{(0)}))$
\ENDFOR
\STATE $\mathbf{W}^{(0)} \leftarrow [\mathbf{w}_1^{(0)},\dots,\mathbf{w}_C^{(0)}]^\top$

\vspace{0.3em}
\STATE \textcolor{blue}{\textbf{// Closed-Loop Inference per View}}
\FOR{$v=1$ \textbf{to} $V$}
    \STATE \textcolor{gray}{\it \# Step 1: Anchor-Guided Visual Encoding}
    \STATE $\mathbf{P}^{(0)} \leftarrow \Phi_{\mathrm{T2V}}(\mathbf{W}^{(0)})$
    \STATE $\mathbf{z}^{(0)}_v \leftarrow \mathrm{Normalize}(f_{\mathcal V}(\mathbf{x}^{(v)}\mid \mathbf{P}^{(0)}))$

    \FOR{$k=0$ \textbf{to} $N-1$}
        \STATE \textcolor{gray}{\it \# Step 2: V2T Refinement}
        \STATE $\boldsymbol{\delta}^{(k+1)}_v \leftarrow \Theta_{\mathrm{V2T}}(\mathbf{z}^{(k)}_v)$
        \FORALL{$c \in \mathcal{C}$}
             \STATE $t^{(k+1)}_{c,v} \leftarrow \mathrm{Compose}(t_c^{(0)}, \boldsymbol{\delta}^{(k+1)}_v)$
             \STATE $\mathbf{w}^{(k+1)}_{c,v} \leftarrow \mathrm{Normalize}(f_{\mathcal T}(t^{(k+1)}_{c,v}))$
        \ENDFOR
        \STATE $\mathbf{W}^{(k+1)}_v \leftarrow [\mathbf{w}^{(k+1)}_{1,v},\dots,\mathbf{w}^{(k+1)}_{C,v}]^\top$

        \STATE \textcolor{gray}{\it \# Step 1: T2V Re-Projection (Loop)}
        \STATE $\mathbf{P}^{(k+1)}_v \leftarrow \Phi_{\mathrm{T2V}}(\mathbf{W}^{(k+1)}_v)$
        \STATE $\mathbf{z}^{(k+1)}_v \leftarrow \mathrm{Normalize}(f_{\mathcal V}(\mathbf{x}^{(v)}\mid \mathbf{P}^{(k+1)}_v))$
    \ENDFOR
    \STATE $\boldsymbol{\ell}_v \leftarrow s \cdot \mathbf{z}^{(N)}_v (\mathbf{W}^{(N)}_v)^\top$
\ENDFOR

\vspace{0.3em}
\STATE \textcolor{blue}{\textbf{// Multi-View Aggregation}}
\FOR{$i=1$ \textbf{to} $V$}
    \FOR{$j=1$ \textbf{to} $V$}
        \STATE $\mathbf{S}_{ij} \leftarrow \langle \mathbf{z}^{(N)}_i, \mathbf{z}^{(N)}_j \rangle$
    \ENDFOR
    \STATE $\mathrm{score}_i \leftarrow \frac{1}{k}\sum_{j \in \mathrm{TopK}_k(\mathbf{S}_{i,:})} \mathbf{S}_{ij}$
\ENDFOR
\STATE $\alpha \leftarrow \mathrm{Softmax}(\mathrm{score} / \tau_{\mathrm{mv}})$
\STATE $\boldsymbol{\ell}_{\mathrm{agg}} \leftarrow \sum_{v=1}^V \alpha_v \, \boldsymbol{\ell}_v$

\RETURN $\arg\max_c (\boldsymbol{\ell}_{\mathrm{agg}})_c$
\end{algorithmic}
\end{algorithm}

\subsection{White-box Gradient Flow Through All Components}
\label{app:gradient_flow}
\label{app:grad}

We clarify that gradients can propagate through all components of CLBP, ensuring a genuine white-box setting.

\subsubsection{End-to-end differentiability of the closed loop.}
Every operation in Steps 0--2 is differentiable in its input: $f_{\mathcal V}$ in pixel space, $\Phi_{\mathrm{T2V}}$ as a stack of linear, normalisation, and attention blocks, $\Theta_{\mathrm{V2T}}$ as an MLP, and $f_{\mathcal T}$ in its continuous prompt embeddings.
The adversarial loss $\mathcal{L}(\boldsymbol{\ell}(\mathbf{x}),y)$ therefore admits the standard chain-rule gradient $\nabla_{\mathbf{x}}\mathcal{L} = (\partial\mathcal{L}/\partial\boldsymbol{\ell})(\partial\boldsymbol{\ell}/\partial\mathbf{z}^{(N)})(\partial\mathbf{z}^{(N)}/\partial\mathbf{x})$, with the last factor expanded through the $N$ loop iterations by autograd.

\subsubsection{Gradients through multi-view aggregation.}
For fixed views $\{\mathbf{x}^{(v)}\}_{v=1}^V$, the aggregated logits $\boldsymbol{\ell}_{\mathrm{agg}}\!=\!\sum_v\alpha_v\boldsymbol{\ell}_v$ (Eq.~\eqref{eq:mv_main}) are differentiable in each $\mathbf{x}^{(v)}$ except on the measure-zero set where the discrete $\mathrm{TopK}$ index set changes.
On a backward pass we treat the selected $\mathcal{N}_i$ as fixed; autograd then routes gradients through both the explicit dependence of $\boldsymbol{\ell}_v$ on $\mathbf{x}^{(v)}$ and the implicit dependence of $\alpha_r$ on $\mathbf{x}^{(v)}$ via $\mathrm{sc}_r$ and $\mathbf{S}_{ij}\!=\!\langle\mathbf{z}_i,\mathbf{z}_j\rangle$, yielding the full $\nabla_{\mathbf{x}^{(v)}}\mathcal{L}$ without truncation.

\subsubsection{No gradient masking}
We do not rely on non-differentiable defenses or randomized obfuscation in the forward path.
In training, the clean branch used for distillation is detached (stop-gradient) as a \emph{target},
but the adversarial branch remains fully differentiable, so attacks that maximise $\mathcal{L}(\mathrm{CLBP}(\mathbf{x}+\boldsymbol{\delta}),y)$ over $\|\boldsymbol{\delta}\|_p\!\le\!\epsilon$ (the threat model in Section~\ref{sec:prelim})
receive valid first-order gradients w.r.t.\ $\mathbf{x}$.

\label{app:train}

\subsubsection{Loss definition}
Let $\boldsymbol{\ell}^{\mathrm{adv}}_{\mathrm{agg}}, \boldsymbol{\ell}^{\mathrm{clean}}_{\mathrm{agg}}$ be the aggregated logits on adversarial and clean inputs (with the clean branch detached as a stop-gradient target).
The training loss is the sum of three terms,
\begin{equation}
  \mathcal{L} = \mathcal{L}_{\mathrm{ce}}(\boldsymbol{\ell}^{\mathrm{adv}}_{\mathrm{agg}},y)
  + \lambda_{\mathrm{mar}}\big[\max_{c\neq y}\ell^{\mathrm{adv}}_c - \ell^{\mathrm{adv}}_y + \gamma\big]_+
  + \lambda_{\mathrm{kl}}\,\mathrm{KL}\!\big(\sigma_T(\boldsymbol{\ell}^{\mathrm{clean}}_{\mathrm{agg}}) \,\big\|\, \sigma_T(\boldsymbol{\ell}^{\mathrm{adv}}_{\mathrm{agg}})\big),
\end{equation}
where $\sigma_T(\cdot)$ denotes the softmax with temperature $T$.

\subsubsection{Default hyperparameters and optimisation}
\label{app:optimization}
We set $\gamma{=}1$, $T{=}2$, $\lambda_{\mathrm{mar}}{=}\lambda_{\mathrm{kl}}{=}1$ throughout (unless noted otherwise), and train all learnable parameters with AdamW at initial learning rate $2{\times}10^{-2}$, cosine-annealed over the full schedule, with gradient clipping at $\ell_2$-norm $1.0$.
\section{Additional Experimental Results}
This section reports per-dataset numbers and additional experiments deferred from the main text: the full Base-to-New table across 11 datasets, results on the ViT-B/16 backbone, and AutoAttack evaluations.

\subsection{Base to New Results}
\label{basetonew_detail}


Table~\ref{tab:base2new_appendix} provides the detailed performance metrics for the 16-shot Base-to-New setting. We evaluate both Clean Accuracy and Robustness (under PGD-100 attack) across 11 diverse datasets. 

As observed, previous state-of-the-art methods often trade off new-class generalization for robustness (or vice versa). In contrast, \textbf{CLBP} achieves a superior Pareto frontier. Specifically, on the Average metric, CLBP improves clean accuracy on new classes by a significant margin while maintaining state-of-the-art robustness. This validates that our geometry-aware prompt learning does not overfit to the base classes, a common failure mode in adversarial prompt tuning.

\begin{table}[t!]
\caption{\textbf{Per-dataset adversarial base-to-new generalization (16-shot).} Each dataset is split into Base (training) and New (held-out) categories. All methods are trained on 16-shot Base data for 10 epochs with 2-step PGD at $\epsilon_{\mathrm{train}}{=}1/255$, $\alpha{=}1/255$, and evaluated on Base / New test splits under PGD-100 at $\epsilon_a{=}1/255$ in $\ell_\infty$. We report Top-1 clean and robust accuracy (\%); B\,=\,Base, N\,=\,New.}
\label{tab:base2new_appendix}
\centering
\renewcommand{\arraystretch}{1.15}
\setlength{\tabcolsep}{5pt}
\small
\begin{tabular}{l cc cc cc cc cc}
\toprule
& \multicolumn{2}{c}{AdvVP}
& \multicolumn{2}{c}{AdvVLP}
& \multicolumn{2}{c}{AdvMaPLe}
& \multicolumn{2}{c}{FAP}
& \multicolumn{2}{c}{\textbf{CLBP}} \\
\cmidrule(lr){2-3}\cmidrule(lr){4-5}\cmidrule(lr){6-7}\cmidrule(lr){8-9}\cmidrule(lr){10-11}
Dataset & B & N & B & N & B & N & B & N & B & N \\
\midrule
\multicolumn{11}{@{}l@{}}{\textit{Clean Accuracy (\%)}} \\[1pt]
ImageNet    & 49.87 & 44.80 & 58.40 & 48.83 & 58.47 & 48.67 & 58.10 & 47.83 & \textbf{68.87} & \textbf{58.34} \\
Caltech     & 92.83 & 88.83 & 94.40 & 83.27 & 94.87 & 84.47 & 94.07 & 76.53 & \textbf{98.19} & \textbf{90.50} \\
DTD         & 23.27 & 13.23 & 43.40 & 21.27 & 48.63 & 22.87 & 69.17 & 35.17 & \textbf{75.93} & \textbf{50.24} \\
Pets        & 32.57 & 32.30 & 38.97 & 39.67 & 60.67 & 57.90 & 87.37 & 72.13 & \textbf{93.99} & \textbf{93.79} \\
Food        &  2.27 &  2.20 & 71.37 & 68.93 & 71.40 & 69.90 & 72.37 & 68.20 & \textbf{85.55} & \textbf{85.35} \\
Flowers     & 50.43 & 45.23 & 88.90 & 49.90 & 56.53 & 30.00 & 89.30 & 45.67 & \textbf{93.07} & \textbf{61.99} \\
SUN397      & 60.20 & 62.20 & 70.23 & 63.57 & 70.57 & 63.27 & 68.47 & 61.47 & \textbf{78.35} & \textbf{72.53} \\
UCF101      &  1.77 &  2.47 & 72.77 & 49.83 & 72.80 & 50.70 & 70.37 & 47.10 & \textbf{83.45} & \textbf{63.71} \\
Aircraft    &  2.30 &  2.00 &  9.93 &  6.73 & 15.00 &  9.97 & 24.83 & 15.83 & \textbf{25.39} &  8.34 \\
Cars        & 14.87 & 15.53 & 55.60 & 46.00 & 56.47 & 46.03 & 53.97 & 42.67 & \textbf{72.31} & \textbf{61.85} \\
EuroSAT     & 18.07 & 17.77 & 49.03 & 35.63 & 54.30 & 15.90 & 87.70 & 32.80 & \textbf{91.36} & \textbf{40.97} \\
\midrule
\textit{Avg.} & 31.68 & 30.39 & 60.38 & 46.18 & 58.95 & 46.92 & 70.52 & 49.58 & \textbf{78.77} & \textbf{62.51} \\
\midrule
\addlinespace[2pt]
\multicolumn{11}{@{}l@{}}{\textit{PGD-100 Robustness (\%)}} \\[1pt]
ImageNet    & 12.27 & 12.27 & 25.33 & 21.03 & 24.93 & 20.50 & 25.83 & 21.57 & \textbf{59.45} & \textbf{50.91} \\
Caltech     & 57.17 & 49.13 & 73.90 & 56.70 & 76.23 & 57.67 & 74.20 & 50.00 & \textbf{95.80} & \textbf{86.68} \\
DTD         & 10.03 &  7.20 & 21.50 &  9.97 & 27.57 & 12.37 & 41.63 & 19.77 & \textbf{70.37} & \textbf{44.81} \\
Pets        & 12.27 & 13.37 & 16.80 & 17.50 & 31.80 & 28.90 & 34.13 & 26.07 & \textbf{89.00} & \textbf{88.53} \\
Food        &  1.27 &  1.00 & 27.90 & 24.50 & 28.43 & 24.60 & 27.57 & 24.20 & \textbf{74.41} & \textbf{73.97} \\
Flowers     & 24.63 & 15.77 & 62.80 & 21.07 & 36.70 & 11.63 & 65.50 & 18.10 & \textbf{86.51} & \textbf{53.05} \\
SUN397      & 18.50 & 21.10 & 33.87 & 29.83 & 34.10 & 29.40 & 34.63 & 30.77 & \textbf{70.10} & \textbf{65.21} \\
UCF101      &  1.73 &  1.43 & 36.37 & 20.13 & 36.77 & 18.00 & 36.63 & 18.30 & \textbf{79.01} & \textbf{53.76} \\
Aircraft    &  0.30 &  2.00 &  4.53 &  2.50 &  6.63 &  3.13 &  8.00 &  4.23 & \textbf{22.02} &  7.13 \\
Cars        &  2.77 &  3.70 & 16.97 & 12.67 & 16.57 & 12.10 & 18.60 & 14.10 & \textbf{60.64} & \textbf{49.29} \\
EuroSAT     & 17.77 & 19.97 & 38.03 & 19.47 & 15.90 &  6.83 & 51.80 & 13.40 & \textbf{88.88} & \textbf{30.59} \\
\midrule
\textit{Avg.} & 14.43 & 13.36 & 30.69 & 20.25 & 32.37 & 21.61 & 38.05 & 21.86 & \textbf{72.38} & \textbf{54.92} \\
\bottomrule
\end{tabular}
\end{table}

\subsection{Performance on ViT-B/16 Backbone}
\label{sec:vit16}

To assess the method's adaptability, we extend the evaluation to the ViT-B/16 backbone. This setting introduces two challenges: finer visual granularity (using $16\times16$ patches) and a high-intensity attack budget ($\epsilon=4/255$). As indicated in Table~\ref{tab:comparison_vitb16}, while standard baselines exhibit limited resilience under these conditions, CLBP achieves an average robustness of 48.2\%, compared to 39.9\% for the R-TPT baseline.

The sustained effectiveness of CLBP in this context can be analyzed from a multi-modal perspective. The transition to a finer patch size increases the complexity of the visual feature space. In such a high-dimensional space, uni-modal adaptation might struggle to maintain strict semantic alignment when subjected to strong perturbations. CLBP, by contrast, employs a bidirectional prompting mechanism that enforces a closed-loop interaction between visual and textual modalities. This approach theoretically acts as a mutual regularization constraint: the visual features guide the text, and the text, in turn, recalibrates the vision. This dual-modal coupling likely prevents the semantic decoupling often observed in larger architectures under attack, thereby preserving robustness even with finer architectural granularity.

\begin{table}[t!]
\caption{\textbf{Quantitative results on the ViT-B/16 backbone (zero-shot).} Models trained on ImageNet for 5 epochs with 2-step PGD at $\epsilon_{\mathrm{train}}{=}1/255$, $\alpha{=}1/255$ are evaluated zero-shot on 8 downstream datasets under PGD-100 at $\epsilon_a{=}4/255$, $\alpha{=}1/255$ in $\ell_\infty$ — a finer patch granularity ($16{\times}16$) combined with a stronger attack budget. We report Top-1 clean accuracy (Acc) and adversarial robustness (Rob, \%).}
\label{tab:comparison_vitb16}
\centering
\resizebox{\textwidth}{!}{
\begin{tabular}{lcccccccccccccccccc}
\toprule
 & \multicolumn{2}{c}{Caltech101} & \multicolumn{2}{c}{Pets} & \multicolumn{2}{c}{Cars} & \multicolumn{2}{c}{Flowers} & \multicolumn{2}{c}{Aircraft} & \multicolumn{2}{c}{DTD} & \multicolumn{2}{c}{Eurosat} & \multicolumn{2}{c}{UCF101} & \multicolumn{2}{c}{avg} \\
\cmidrule(lr){2-3} \cmidrule(lr){4-5} \cmidrule(lr){6-7} \cmidrule(lr){8-9} \cmidrule(lr){10-11} \cmidrule(lr){12-13} \cmidrule(lr){14-15} \cmidrule(lr){16-17} \cmidrule(lr){18-19}
Method & Acc & Rob & Acc & Rob & Acc & Rob & Acc & Rob & Acc & Rob & Acc & Rob & Acc & Rob & Acc & Rob & Acc & Rob \\
\midrule
CLIP & 94.0 & 0.0 & 88.3 & 0.0 & 65.5 & 0.0 & 67.4 & 0.0 & 23.9 & 0.0 & 44.4 & 0.0 & 42.2 & 0.0 & 65.2 & 0.0 & 61.4 & 0.0 \\
TPT & 94.1 & 0.0 & 87.4 & 0.0 & 66.5 & 0.0 & 69.1 & 0.0 & 23.4 & 0.0 & 46.9 & 0.0 & 42.6 & 0.0 & 67.9 & 0.0 & 62.2 & 0.0 \\
C-TPT & 93.9 & 0.0 & 88.2 & 0.0 & 65.8 & 0.0 & 69.6 & 0.0 & 23.9 & 0.0 & 45.9 & 0.0 & 42.3 & 0.0 & 65.5 & 0.0 & 61.9 & 0.0 \\
MTA & 94.3 & 72.1 & 88.0 & 51.8 & 67.7 & 18.5 & 67.4 & 27.9 & 25.0 & 4.3 & 46.5 & 16.2 & 42.5 & 1.2 & 67.5 & 27.5 & 62.3 & 27.4 \\
R-TPT & 93.7 & 82.0 & 87.2 & 60.2 & 67.0 & 34.7 & 68.7 & 44.6 & 23.9 & 13.2 & 46.4 & 32.8 & 34.7 & 8.5 & 67.2 & 43.2 & 61.1 & 39.9 \\
CLBP & 93.5 & \textbf{86.1} & 89.2 & \textbf{79.7} & 65.1 & \textbf{42.8} & 70.5 & \textbf{56.0} & 21.8 & \textbf{14.6} & 44.7 & \textbf{37.4} & 34.0 & \textbf{16.6} & 68.2 & \textbf{52.6} & 60.9 & \textbf{48.2} \\
\bottomrule
\end{tabular}
}
\end{table}

\subsection{Robustness against AutoAttack}
\label{sec:autoattack}
\noindent \textbf{Robustness Evaluation.} Following established protocols~\cite{liu2025self, zhou2024few}, we evaluate robustness using the \textbf{AutoAttack}~\cite{croce2020reliable} framework under $L_\infty$constraints with $\epsilon \in \{1/255, 2/255, 4/255\}$. We specifically employ APGD-CE and APGD-DLR to ensure reliable assessment and mitigate gradient masking with manageable computational cost.
\noindent \textbf{Results and Analysis.} As shown in Table~\ref{tab:fewshot_attack_results}, CLBP consistently outperforms baselines on OxfordPets and DTD across all budgets. Notably, CLBP exhibits superior stability at higher perturbations: on OxfordPets with $\epsilon=4/255$, it achieves \textbf{68.93\%} robust accuracy, significantly surpassing the best baseline (FAP, 12.27\%). This confirms that our closed-loop bidirectional prompting effectively enhances intrinsic robustness and generalization.


\begin{table}[t!]
\centering
\caption{\textbf{Robustness against AutoAttack on OxfordPets and DTD.} Models are trained on 16-shot data for 10 epochs with 2-step PGD at $\epsilon_{\mathrm{train}}{=}1/255$, $\alpha{=}1/255$. Following~\cite{liu2025self,zhou2024few}, evaluation uses the AutoAttack ensemble (APGD-CE + APGD-DLR) under $\ell_\infty$ with perturbation budgets $\epsilon_a \in \{1, 2, 4\}/255$. We report Top-1 clean (Acc) and robust (Rob) accuracy (\%); identical clean-Acc within a method indicates a single trained model evaluated under three attack budgets.}
\label{tab:fewshot_attack_results}
\vspace{2mm}
\small
\begin{tabular}{lcccccccc}
\toprule
\multirow{2}{*}{\textbf{Dataset}} & \multirow{2}{*}{\textbf{Method}} & \multicolumn{2}{c}{\textbf{$\epsilon=1$}} & \multicolumn{2}{c}{\textbf{$\epsilon=2$}} & \multicolumn{2}{c}{\textbf{$\epsilon=4$}} \\
\cmidrule(lr){3-4} \cmidrule(lr){5-6} \cmidrule(lr){7-8}
& & \textbf{Acc (\%)} & \textbf{Rob (\%)} & \textbf{Acc (\%)} & \textbf{Rob (\%)} & \textbf{Acc (\%)} & \textbf{Rob (\%)} \\
\midrule
\multirow{5}{*}{OxfordPets} 
& AdvVP & 78.00 & 17.32 & 78.00 & 3.98 & 78.00 & 0.34 \\
& AdvVLP & 79.18 & 26.40 & 79.18 & 19.13 & 79.18 & 9.26 \\
& AdvMaPLe & 78.00  & 26.24 & 78.00 & 19.00 & 78.00 & 9.48 \\
& FAP & 80.02  & 30.19 & 80.02  & 22.82 & 80.02 & 12.27 \\
& CLBP & \textbf{85.04} & \textbf{78.60} & \textbf{85.04} & \textbf{75.20} & \textbf{85.04} & \textbf{68.93} \\
\midrule
\multirow{5}{*}{DTD} 
& AdvVP & 32.74 & 8.62 & 32.74 & 5.62 & 32.74 & 2.07 \\
& AdvVLP & 32.57 & 13.62 & 32.57 & 12.50 & 32.57 & 10.02 \\
& AdvMaPLe & 28.31 & 13.34 & 28.31 & 11.89 & 28.31 & 9.47 \\
& FAP & 30.56 & 14.29 & 30.56 & 12.87 & 30.56 & 10.47 \\
& CLBP & \textbf{44.86} & \textbf{38.95} & \textbf{44.86} & \textbf{36.41} & \textbf{44.86} & \textbf{31.86} \\
\bottomrule
\end{tabular}
\end{table}

\subsection{Robustness against CW Attack}
\label{app:cw}

We additionally evaluate CLBP under a margin-based Carlini--Wagner attack~\cite{carlini2017towards} adapted to the $\ell_\infty$ threat model with $\epsilon{=}1/255$, following the protocol of prior CLIP-robustness work.
Table~\ref{tab:cw_results} reports clean and robust accuracy on 8 datasets across 9 baselines spanning vanilla CLIP, adversarial fine-tuning (CLIP-FT, PMG-AFT, FARE, TeCoA), test-time defences (Anti-adv, TTC, SCC), and CLBP.
CLBP attains the highest robust accuracy on every reported dataset; the average margin over SCC, the strongest baseline, is $+10.0\%$, with the largest single-dataset gap on Caltech101 ($+10.35\%$) and OxfordPets ($+5.34\%$).
The clean accuracy of CLBP is also comparable to or above the strongest test-time baseline on most datasets, indicating that the robustness gain under CW attack is not paid for by a clean-accuracy drop.

\begin{table}[h!]
\centering
\caption{\textbf{Robustness under the CW-$\ell_\infty$ attack.} All training-based baselines follow the cross-dataset transfer protocol (trained on ImageNet for 5 epochs with 2-step PGD at $\epsilon_{\mathrm{train}}{=}1/255$, $\alpha{=}1/255$) and are evaluated zero-shot on 8 downstream datasets under the margin-based Carlini--Wagner attack at $\epsilon_a{=}1/255$ in $\ell_\infty$. We report Top-1 clean / robust accuracy (\%); the best robust accuracy in each row is in \textbf{bold}.}
\label{tab:cw_results}
\small
\setlength{\tabcolsep}{2pt}
\renewcommand{\arraystretch}{1.05}
\resizebox{\textwidth}{!}{%
\begin{tabular}{l ccccccccc}
\toprule
\multirow{2}{*}{\textbf{Dataset}}
  & CLIP        & CLIP-FT     & PMG-AFT     & FARE
  & TeCoA       & Anti-adv    & TTC         & SCC
  & \textbf{CLBP (Ours)} \\
  & Clean / Rob & Clean / Rob & Clean / Rob & Clean / Rob
  & Clean / Rob & Clean / Rob & Clean / Rob & Clean / Rob
  & Clean / Rob \\
\midrule
StanfordCars & 52.02 / 2.38  & 42.11 / 2.04  & 25.44 / 10.53 & 38.68 / 9.14
             & 20.91 / 8.74  & 36.21 / 4.76  & 48.16 / 30.38 & 51.19 / 37.96
             & \textbf{60.05} / \textbf{47.68} \\
Caltech101   & 85.66 / 20.88 & 83.63 / 15.95 & 75.45 / 61.58 & 80.95 / 54.86
             & 71.68 / 56.23 & 84.02 / 41.47 & 86.53 / 66.17 & 86.44 / 76.59
             & \textbf{91.12} / \textbf{86.94} \\
DTD          & 40.64 / 2.87  & 36.49 / 2.77  & 21.76 / 13.72 & 32.07 / 14.36
             & 25.16 / 16.28 & 38.92 / 6.06  & 36.98 / 27.39 & 37.34 / 33.35
             & \textbf{44.86} / \textbf{38.83} \\
Flowers102   & 65.46 / 1.35  & 53.37 / 0.80  & 37.00 / 21.34 & 47.98 / 17.25
             & 36.80 / 21.13 & 62.66 / 8.06  & 64.16 / 36.84 & 64.16 / 49.76
             & 63.58 / \textbf{54.61} \\
OxfordPets   & 87.44 / 1.64  & 84.14 / 1.14  & 65.88 / 39.28 & 79.37 / 33.85
             & 62.12 / 37.91 & 80.62 / 22.99 & 83.35 / 57.15 & 86.48 / 75.06
             & 85.04 / \textbf{80.40} \\
ImageNet     & 59.69 / 1.46  & 54.24 / 1.27  & 36.12 / 19.42 & 48.79 / 27.71
             & 34.89 / 18.28 & 54.27 / 9.37  & 49.49 / 36.01 & 56.03 / 45.75
             & \textbf{65.96} / \textbf{55.24} \\
Food101      & 84.88 / 1.09  & 64.86 / 0.55  & 36.61 / 16.57 & 55.31 / 12.93
             & 29.98 / 12.87 & 75.81 / 15.03 & 82.18 / 54.65 & 82.13 / 59.73
             & 77.46 / \textbf{65.22} \\
SUN397       & 58.50 / 1.75  & 55.73 / 1.48  & 37.98 / 20.39 & 52.42 / 15.73
             & 36.69 / 18.36 & 56.00 / 8.85  & 55.13 / 39.44 & 58.25 / 48.99
             & \textbf{65.24} / \textbf{54.26} \\
\midrule
\textbf{Avg.}
             & 66.79 / 4.18  & 59.32 / 3.25  & 42.03 / 25.35 & 54.45 / 23.23
             & 39.78 / 23.73 & 61.06 / 14.57 & 63.25 / 43.50 & 65.25 / 53.40
             & \textbf{69.16} / \textbf{60.40} \\
\bottomrule
\end{tabular}}
\end{table}

\subsection{Adaptive EOT-PGD Evaluation}
\label{app:eot}

The multi-view aggregation in Eq.~\eqref{eq:mv_main} introduces stochasticity through the per-view augmentations and the TopK consistency selection.
Following the adaptive-attack protocol introduced in Section~\ref{sec:experiments}, we estimate expected gradients over the augmentation distribution at eot$\,\in\!\{1,4\}$~\cite{athalye2018obfuscated}, against the full stochastic CLBP pipeline (no fixed seed during the backward pass), and additionally report the few-shot setting where each method is tuned on $16$ shots per class.
The robust accuracy is reported under PGD-100 with $\epsilon{=}1/255$.

\begin{table}[h!]
  \centering
  \small
  \setlength{\tabcolsep}{8pt}
  \renewcommand{\arraystretch}{1.1}
  \caption{\textbf{Adaptive EOT-PGD~\cite{athalye2018obfuscated} on the unaltered stochastic CLBP pipeline.} Both methods are evaluated under PGD-100 at $\epsilon_a{=}1/255$ in $\ell_\infty$, with gradients estimated via Expectation-over-Transformation across $\mathrm{eot}{\in}\{1,4\}$ augmentation samples. \emph{Zero-shot} rows use the cross-dataset transfer model (trained on ImageNet, 5 epochs, 2-step PGD at $\epsilon_{\mathrm{train}}{=}1/255$); \emph{16-shot} rows use the per-dataset 16-shot adapted model (10 epochs, same attack). Highlighted cells correspond to the strongest adaptive setting $\mathrm{eot}{=}4$.}
  \label{tab:eot_appendix}
  \begin{tabular}{l l cc cc}
    \toprule
    \multirow{2}{*}{Setting} & \multirow{2}{*}{Method}
        & \multicolumn{2}{c}{DTD} & \multicolumn{2}{c}{Caltech101} \\
    \cmidrule(lr){3-4}\cmidrule(lr){5-6}
    & & eot=1 & eot=4 & eot=1 & eot=4 \\
    \midrule
    \multirow{2}{*}{zero-shot}
        & R-TPT  & 15.72 & \cellcolor{HiRed}12.77
                 & 51.27 & \cellcolor{HiRed}45.80 \\
        & \textbf{CLBP} & \textbf{34.28} & \cellcolor{HiRed}\textbf{28.72}
                        & \textbf{79.55} & \cellcolor{HiRed}\textbf{73.14} \\
    \midrule
    \multirow{2}{*}{few-shot}
        & R-TPT  & 20.69 & \cellcolor{HiRed}17.31
                 & 65.92 & \cellcolor{HiRed}57.44 \\
        & \textbf{CLBP} & \textbf{22.93} & \cellcolor{HiRed}\textbf{19.28}
                        & \textbf{68.64} & \cellcolor{HiRed}\textbf{60.37} \\
    \bottomrule
  \end{tabular}
\end{table}

\noindent\textbf{Discussion.}
Increasing eot from $1$ to $4$ tightens the attack and reduces robust accuracy for both methods, but the relative ordering is preserved across all eight cells: CLBP retains the highest robust accuracy under every (setting, dataset, eot) configuration, with the gap reaching $+27.34\%$ on Caltech101 zero-shot at eot$\,{=}\,4$.
Two observations follow.
First, the gain of CLBP is not contingent on a fixed-seed evaluation: applying EOT against the random augmentation pipeline does not collapse the gap.
Second, the closed-loop bidirectional correction continues to provide a measurable improvement in the few-shot regime, where the static-prompt baseline R-TPT is already strained.
We refer to App.~\ref{app:gradient_flow} for a description of how gradients flow through the TopK selection in the multi-view aggregation, ruling out the standard masking pattern in which discrete selection is made non-differentiable.

\subsection{Sensitivity to Soft-Prompt Length}
\label{app:prompt_length}

The number of learnable context tokens $L$ is the only hyper-parameter shared between CLBP and standard prompt tuning, and capacity-versus-stability arguments in prior APT work suggest it can affect the trade-off between clean and robust accuracy.
We therefore sweep $L\!\in\!\{2,4,6,8\}$ on DTD with all other settings fixed (training: $\epsilon{=}1/255$, PGD-2, $10$ epochs; evaluation: PGD-100 at the same $\epsilon$).

\begin{table}[h!]
\centering
\small
\setlength{\tabcolsep}{14pt}
\renewcommand{\arraystretch}{1.1}
\caption{\textbf{Soft-prompt context length ablation on DTD.} Trained on 16-shot data for 10 epochs with 2-step PGD at $\epsilon_{\mathrm{train}}{=}1/255$, $\alpha{=}1/255$, evaluated under PGD-100 at $\epsilon_a{=}1/255$ in $\ell_\infty$. We report Top-1 clean and robust accuracy (\%) for context lengths $L{\in}\{2,4,6,8\}$.}
\label{tab:prompt_length}
\begin{tabular}{c cc}
  \toprule
  Context length $L$ & Clean (\%) & Rob.\ (\%) \\
  \midrule
  $2$ & 64.36          & 57.51          \\
  $4$ & 62.88          & 57.57          \\
  $6$ & \textbf{65.66} & \textbf{60.22} \\
  $8$ & 64.72          & 59.22          \\
  \bottomrule
\end{tabular}
\end{table}

Both clean and robust accuracy stay within a $\sim\!3\%$ band, with $L{=}6$ marginally best and $L{=}2$ already competitive.
This indicates that CLBP is not sensitive to the specific prompt length, and the trend is consistent with the capacity-versus-stability trade-off observed in standard prompt tuning: very short prompts under-fit, longer prompts saturate.
We retain $L{=}4$ as a conservative default in all main experiments to align with the protocol used by APT baselines.

\section{Analysis of Representation Geometry Drift}
\label{sec:cka_drift}

To understand how our defense mechanism affects the internal feature representations of the CLIP image encoder, we employ Centered Kernel Alignment (CKA) to measure layer-wise similarity. Figure~\ref{figs:cka_drift} visualizes the Linear CKA scores between the original pre-trained encoder and models fine-tuned with different methods.

A sharp decline in CKA scores, particularly in the middle transformer blocks (Layers 4--8), indicates a significant departure from the original pre-trained geometry. This phenomenon is most pronounced in methods using larger perturbation budgets (e.g., TeCoA-eps4). CLBP maintains a higher CKA score in early layers, suggesting that it preserves low-level feature extraction capabilities while adapting the high-level semantic layers for robustness.

\begin{figure}[t!] 
    \centering
    
    \includegraphics[width=0.6\linewidth]{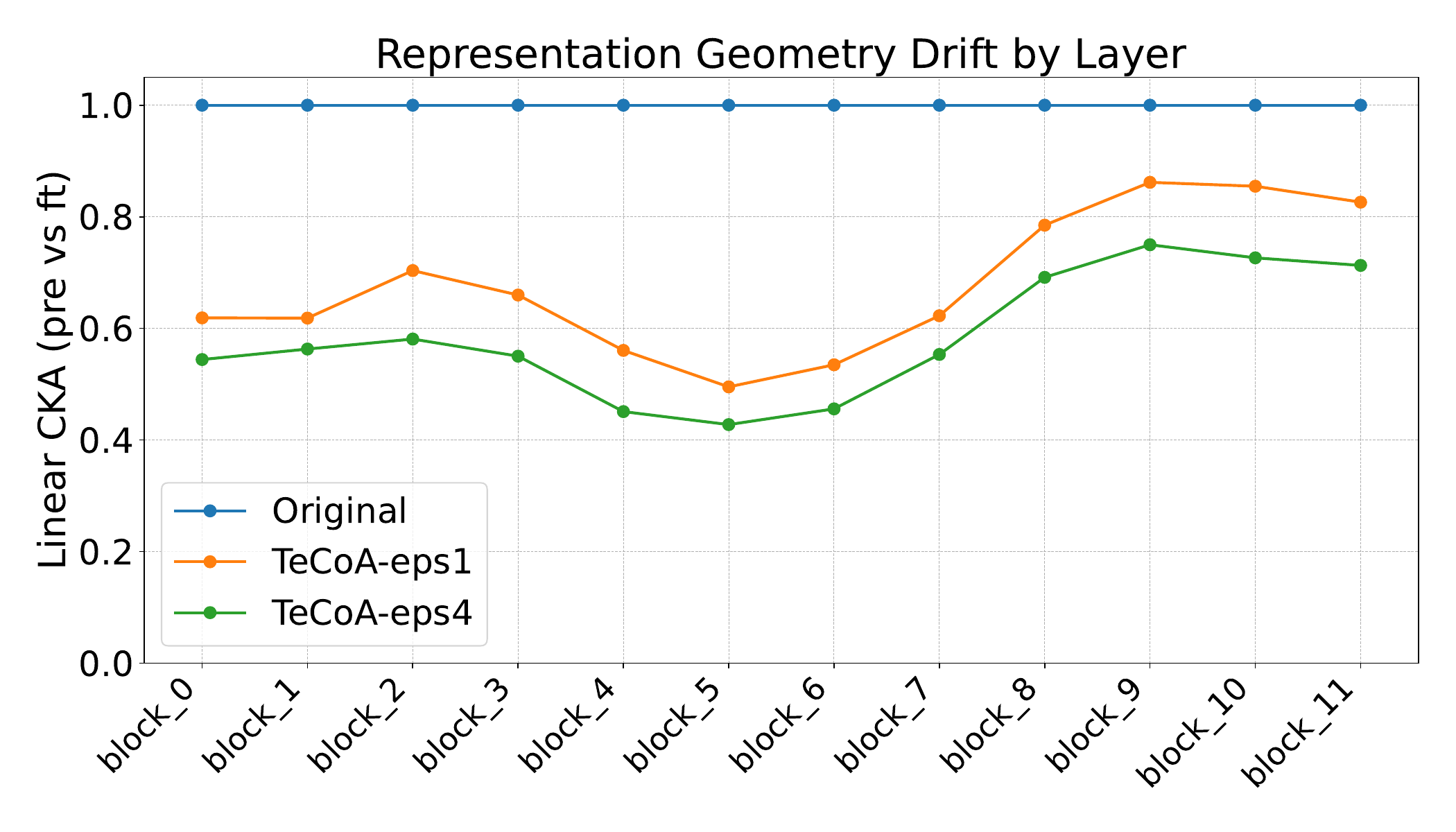}
    
    \caption{\textbf{Representation geometry drift by layer.} Layer-wise feature consistency measured by Linear Centered Kernel Alignment (CKA) between the pre-trained CLIP encoder and adversarially fine-tuned variants on ImageNet. The drop in similarity in middle layers (blocks 4--8) indicates geometric distortion introduced during robust fine-tuning, more pronounced under stronger training budgets ($\epsilon_{\mathrm{train}}{=}4/255$).}
    \label{figs:cka_drift}
    
\end{figure}


\section{Detailed Analysis of Defense Mechanism}
\label{sec:appendix_analysis}

This appendix examines how CLBP affects the geometry of the predictions: the logit margin (\S\ref{subsec:semantic_rectification}) and the internal feature distance (\S\ref{subsec:internal_stability}).

\subsection{Semantic Rectification and Margin Recovery}
\label{subsec:semantic_rectification}

We measure two diagnostics: (i) the \emph{logit margin}, the difference between the logit of the correct class and the largest competitor, and (ii) the cosine similarities between the image feature and each text prototype.
Adversarial attacks collapse the logit margins, often pushing them below zero, while CLBP restores the margin distribution to and beyond the clean baseline (Fig.~\ref{fig:semantic_analysis}, left).
The cosine-similarity decomposition (Fig.~\ref{fig:semantic_analysis}, right) shows that this restoration comes from an increased similarity to the correct prototype together with a decreased similarity to competing prototypes, rather than from a simple denoising of the visual feature.

\begin{figure}[t!]
    \centering

    \begin{subfigure}[t]{0.49\columnwidth}
        \centering
        \includegraphics[height=4.2cm, keepaspectratio]{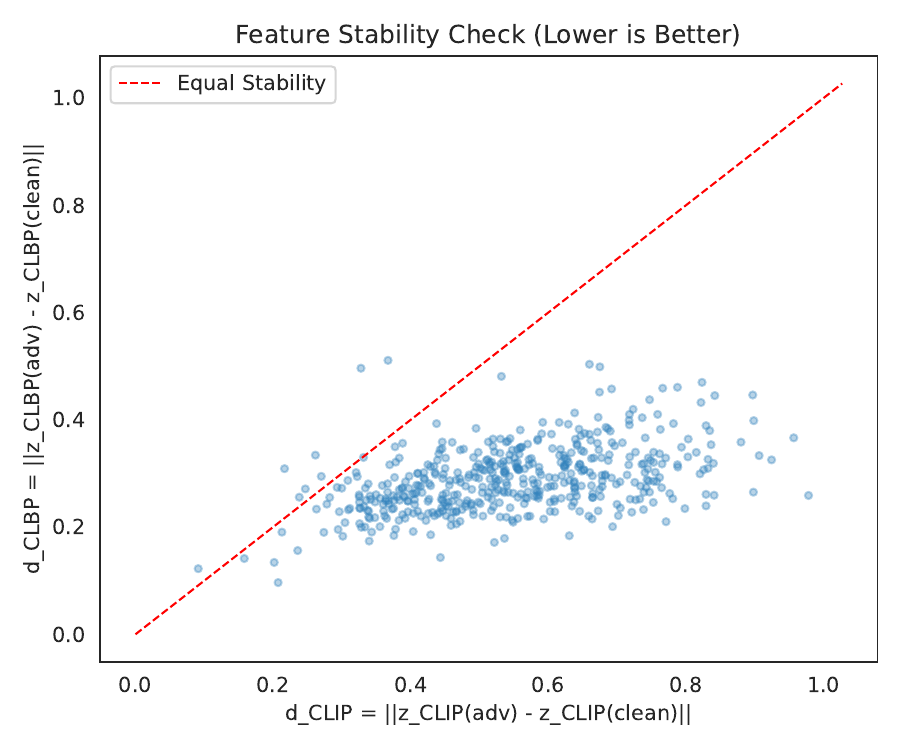}
        \caption{Pairwise internal distance comparison.}
        \label{fig:stability_scatter}
    \end{subfigure}
    \hfill
    \begin{subfigure}[t]{0.49\columnwidth}
        \centering
        \includegraphics[height=4.2cm, keepaspectratio]{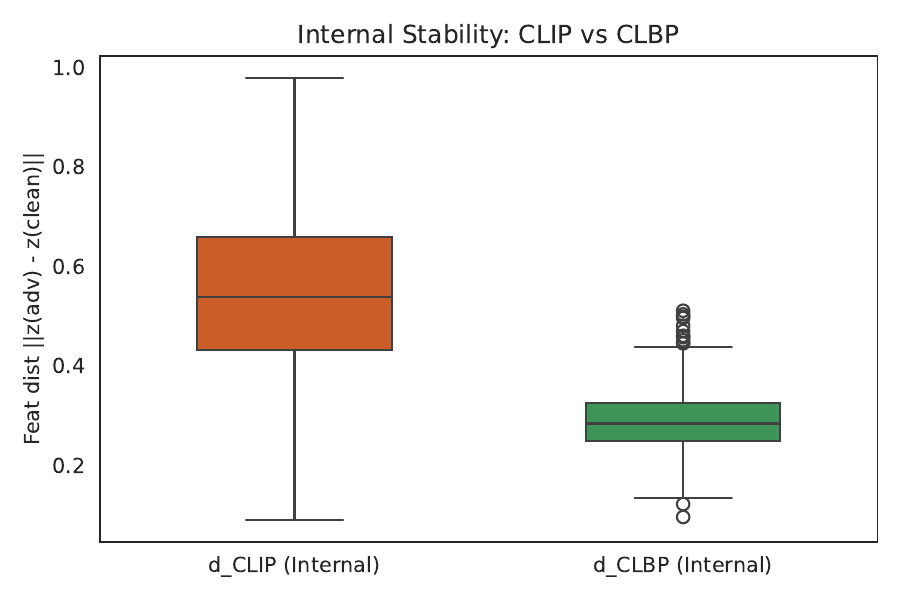}
        \caption{Distribution of feature shifts.}
        \label{fig:stability_box}
    \end{subfigure}

    \caption{\textbf{Internal feature stability under adversarial perturbations.}
    We compare $\ell_2$ distances $\Delta_{\mathcal M}{=}\|\mathbf{z}_{\mathcal M}(\mathbf{x}^{\mathrm{adv}})-\mathbf{z}_{\mathcal M}(\mathbf{x}^{\mathrm{clean}})\|_2$ between clean and adversarial visual features within the same model $\mathcal M\in\{\mathrm{CLIP},\mathrm{CLBP}\}$, with $\mathbf{x}^{\mathrm{adv}}$ generated by PGD-100 at $\epsilon_a{=}1/255$ in $\ell_\infty$ on DTD.
    Most samples satisfy $\Delta_{\mathrm{CLBP}}{<}\Delta_{\mathrm{CLIP}}$ (points below the diagonal in (a)), and CLBP reduces both the median and variance of feature shifts (b), demonstrating a damping effect and improved local smoothness around natural images.}
    \label{fig:stability_analysis}
\end{figure}

\begin{figure}[t!]
    \centering

    \begin{subfigure}[t]{0.39\linewidth} 
        \centering
        \includegraphics[width=\linewidth]{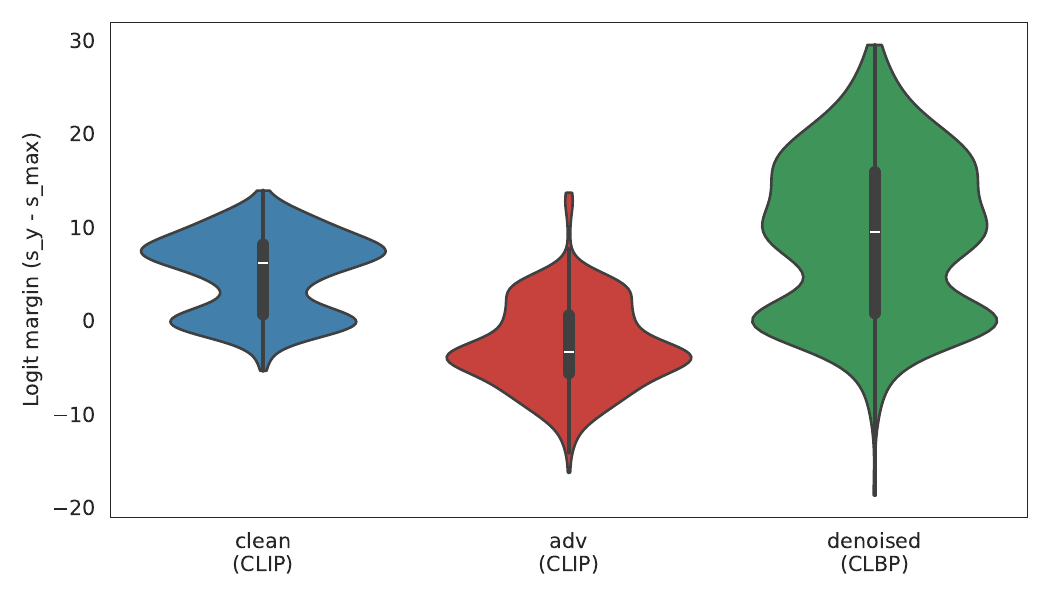}
        \caption{Logit margins.} 
        \label{fig:margin}
    \end{subfigure}
    \hfill
    \begin{subfigure}[t]{0.59\linewidth}
        \centering
        \includegraphics[width=\linewidth]{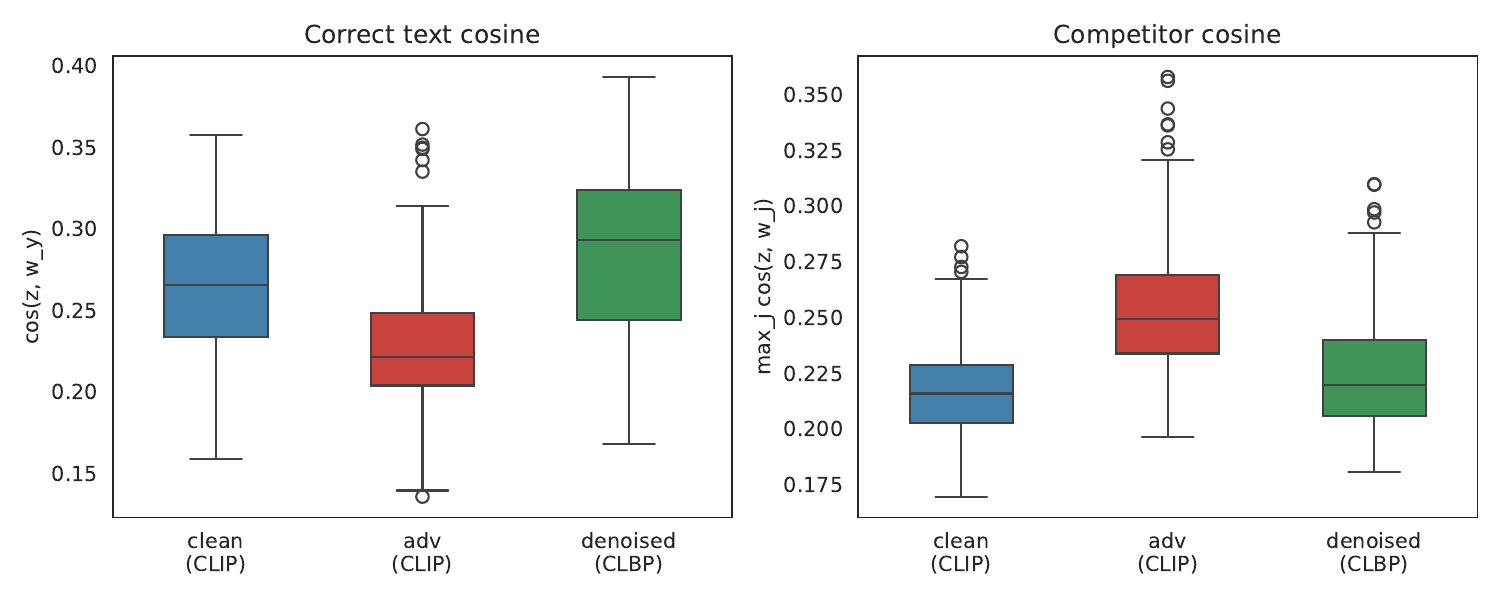}
        \caption{Cosine similarity decomposition.}
        \label{fig:cosine}
    \end{subfigure}



    \caption{\textbf{Semantic rectification analysis on DTD.} 16-shot CLBP (10 epochs, 2-step PGD at $\epsilon_{\mathrm{train}}{=}1/255$) evaluated under PGD-100 at $\epsilon_a{=}1/255$ in $\ell_\infty$.
    (a) Logit margin distributions $\ell_y - \max_{c\neq y}\ell_c$ over the test set: clean (CLIP), adversarial (CLIP), and adversarial after the closed-loop denoising (CLBP).
    (b) Cosine-similarity decomposition into the correct-class similarity $\cos(\mathbf{z},\mathbf{w}_y)$ and the strongest competitor similarity $\max_{j\neq y}\cos(\mathbf{z},\mathbf{w}_j)$.
    \textbf{Interpretation:} CLBP restores margins and re-aligns features rather than reconstructing pixels, consistent with semantic rectification.}
    \label{fig:semantic_analysis}
\end{figure}

\subsection{Internal Feature Stability}
\label{subsec:internal_stability}

We measure the sensitivity of the vision encoder to input perturbations through the internal perturbation distance $\Delta_{(\cdot)}=\|z_{(\cdot)}(x_{\mathrm{adv}})-z_{(\cdot)}(x_{\mathrm{clean}})\|_2$ within the same model, computed for vanilla CLIP and CLBP.
Figure~\ref{fig:stability_analysis} shows that $\Delta_{\mathrm{CLBP}}\!<\!\Delta_{\mathrm{CLIP}}$ on most samples (points below the diagonal), and $\Delta_{\mathrm{CLBP}}$ has lower median and variance.
This is consistent with a smaller effective local Lipschitz constant around natural images, in line with the analysis of $\mathcal{T}_{\mathbf{x}}$ in Section~\ref{sec:theory_main}.

\section{Multi-View Aggregation: View-Count Ablation}
\label{app:viewcount}

We expand each input image into $V$ random augmentations $\{\mathbf{x}^{(v)}\}_{v=1}^V$ and aggregate per-view logits using the consistency-weighted softmax in Eq.~\eqref{eq:mv_main}.
This appendix reports how robustness and latency depend on $V$ on DTD under PGD-100 with $\epsilon{=}4/255$ (training settings: $\epsilon{=}1/255$, $\alpha{=}1/255$, PGD-2, $10$ epochs).

\begin{figure}[t!]
  \centering
  \includegraphics[width=0.7\linewidth]{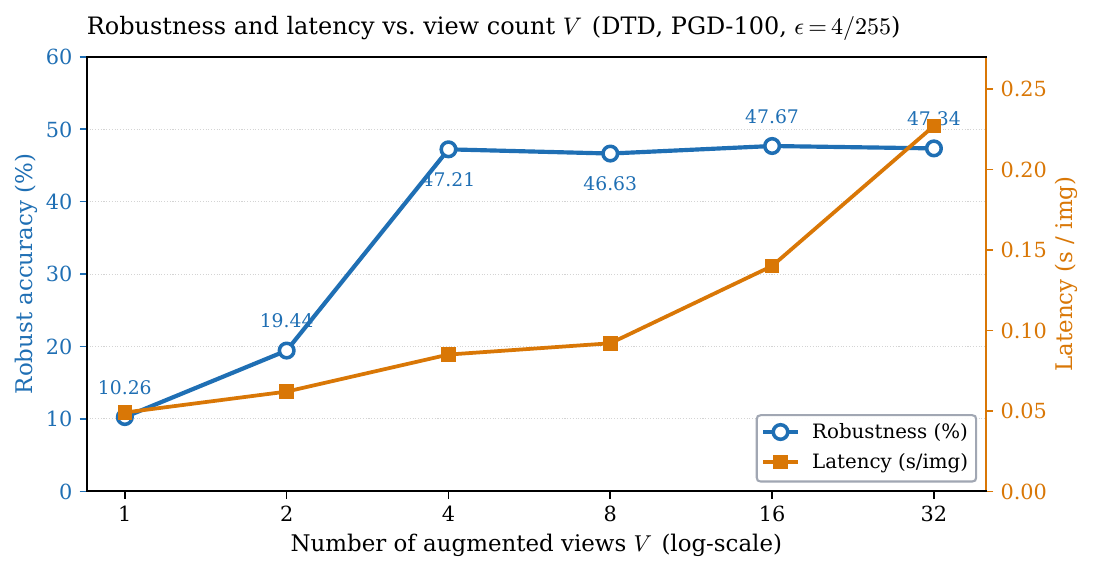}
  \caption{\textbf{Robust accuracy and per-image latency of CLBP on DTD as a function of the augmented-view count $V$.} Trained on 16-shot data with 2-step PGD at $\epsilon_{\mathrm{train}}{=}1/255$, $\alpha{=}1/255$ for $10$ epochs; evaluated under PGD-100 at the stronger $\epsilon_a{=}4/255$ in $\ell_\infty$. All six points are measured on the same trained model.}
  \label{figs:plot_views}
\end{figure}

\paragraph{Effect of $V$.}
The robustness curve in Fig.~\ref{figs:plot_views} is non-monotonic in shape but follows a simple pattern.
At $V{=}1$, the TopK weighting reduces to a single-view prediction and gives $10.26\%$.
At $V{=}2$ the score is $19.44\%$.
At $V{=}4$, which matches the default $K{=}4$ in $\mathrm{TopK}(\mathbf{S}_{i,:})$, the score reaches $47.21\%$.
For $V\!\in\!\{8,16,32\}$ the scores are $46.63\%$, $47.67\%$, and $47.34\%$, all within a $1.1\%$ band, while latency grows from $0.092$ to $0.227$\,s/img.
The flatness of the curve beyond $V{=}8$ is consistent with the outlier-suppression bound in Appendix~\ref{app:morelemmas}.

\paragraph{Sensitivity to test-time sampling.}
The non-monotonic ordering $V{=}16{>}V{=}32$ in Fig.~\ref{figs:plot_views} could reflect test-time augmentation sampling rather than a structural property of $V$.
We re-evaluated the same trained model at $V\!\in\!\{16,32\}$ with three independent random seeds (Tab.~\ref{tab:view_seed}).
Per-seed values vary by up to $0.7\%$, the ranges overlap, and the seed-averaged means differ by $0.09\%$.

\begin{table}[h!]
  \centering
  \caption{\textbf{Sensitivity of the view aggregation to test-time augmentation seeds.} The same 16-shot CLBP model (DTD, 10 epochs, 2-step PGD at $\epsilon_{\mathrm{train}}{=}1/255$) is evaluated under PGD-100 at $\epsilon_a{=}4/255$ in $\ell_\infty$ with three independent random seeds for view sampling. We report Top-1 robust accuracy (\%) for $V{\in}\{16,32\}$.}
  \label{tab:view_seed}
  \small
  \setlength{\tabcolsep}{8pt}
  \begin{tabular}{l c c c c c}
    \toprule
    $V$ & seed 0 & seed 1 & seed 2 & mean & std \\
    \midrule
    $16$ & $47.67$ & $46.99$ & $47.70$ & $47.45$ & $0.40$ \\
    $32$ & $47.34$ & $47.28$ & $47.46$ & $47.36$ & $0.09$ \\
    \bottomrule
  \end{tabular}
\end{table}

\paragraph{Shared-pool evaluation.}
The runs above use disjoint samples of views for $V{=}16$ and $V{=}32$.
We also generated a single pool of $32$ augmentations per image and evaluated the same model on the first $8$, $16$, or all $32$ views of that pool (Tab.~\ref{tab:view_pool}).
Under this nested protocol the gap from $V{=}16$ to $V{=}32$ is $0.06\%$, while latency grows $2.5\times$.

\begin{table}[h!]
  \centering
  \caption{\textbf{Shared-pool view ablation on DTD.} A single pool of $32$ augmented views per image is generated, and the first $V$ views are used for prediction. The same 16-shot CLBP model (10 epochs, 2-step PGD at $\epsilon_{\mathrm{train}}{=}1/255$) is evaluated under PGD-100 at $\epsilon_a{=}4/255$.}
  \label{tab:view_pool}
  \small
  \setlength{\tabcolsep}{10pt}
  \begin{tabular}{l c c c}
    \toprule
    $V$ (first $V$ from a shared pool) & $8$ & $16$ & $32$ \\
    \midrule
    Robust accuracy (\%) & $47.28$ & $\mathbf{47.40}$ & $47.34$ \\
    Latency (s/img) & $\mathbf{0.092}$ & $0.140$ & $0.227$ \\
    \bottomrule
  \end{tabular}
\end{table}

\paragraph{Choice of $V$ for main experiments.}
The main-paper results are reported at $V{=}32$.
This choice is governed by two considerations.
First, the multi-view test-time defences against which we benchmark (R-TPT, SCC) evaluate at $V{=}32$, and using the same view budget rules out a confound in which a smaller $V$ for CLBP would be compared against a larger $V$ for the baselines.
Second, the analysis above shows that CLBP's robust accuracy on DTD is saturated for $V\!\geq\!8$: the spread across $V\!\in\!\{8,16,32\}$ is within $1.1\%$, and the seed-control and shared-pool experiments above place this spread within test-time sampling noise.
A consequence of this saturation is that the main-table numbers reported at $V{=}32$ are not driven by the larger view budget; the same numbers, up to seed-level fluctuation, would be obtained at $V{=}8$ at $0.092$\,s/img instead of $0.227$\,s/img.
We therefore present $V{=}32$ as the conservative protocol setting for the main results, and $V{=}8$ as a deployment-friendly alternative when inference cost is the binding constraint.

\section{Efficiency Analysis}
\label{sec:efficiency}

We study how the closed-loop iteration depth $T$ trades latency for robustness, with the number of augmented views fixed at $V{=}32$ throughout this section (the dependence on $V$ is studied separately in Appendix~\ref{app:viewcount}).
Table~\ref{tab:efficiency} reports per-image latency and PGD-100 robustness ($\epsilon{=}4/255$) on DTD for $T{=}1$ and $T{=}3$.
Going from $T{=}1$ to $T{=}3$ raises robust accuracy by $0.71\%$ and roughly triples the latency, so we use $T{=}1$ in all main experiments.

\begin{table}[t!]
    \centering
    \caption{\textbf{Per-image inference latency vs.\ PGD-100 robust accuracy of CLBP on DTD as a function of the closed-loop iteration depth $T$.} Trained on 16-shot data (10 epochs, 2-step PGD at $\epsilon_{\mathrm{train}}{=}1/255$); evaluated under PGD-100 at $\epsilon_a{=}4/255$. The number of augmented views is fixed at $V{=}32$.}
    \label{tab:efficiency}
    \resizebox{0.55\linewidth}{!}{
    \begin{tabular}{lcc}
        \toprule
        \textbf{Configuration} & \textbf{Latency (s/img)} & \textbf{Robustness (\%)} \\
        \midrule
        CLBP, $T{=}1$ & \textbf{0.227} & 45.98 \\
        CLBP, $T{=}3$ & 0.637 & \textbf{46.69} \\
        \bottomrule
    \end{tabular}
    }
    \vspace{-0.1cm}
\end{table}

    
    

\section{Proofs and Additional Lemmas}
\label{app:proofs}

\subsection{Setup: Reachable Set and Regularity Assumptions}
\label{app:setup}

We recall the closed-loop operator from Eq.~\eqref{eq:Tx}: $\mathcal{T}_{\mathbf{x}} \!\triangleq\! \mathcal{H}_{\mathbf{x}}\!\circ\!\mathcal{G}$ with $\mathbf{z}^{(k+1)}\!=\!\mathcal{T}_{\mathbf{x}}(\mathbf{z}^{(k)})$, where $\mathcal{G}\!:\!\mathbb{S}^{d-1}\!\to\!\mathbb{R}^{C\times d}$ and $\mathcal{H}_{\mathbf{x}}\!:\!\mathbb{R}^{C\times d}\!\to\!\mathbb{S}^{d-1}$ are defined in Section~\ref{sec:theory_main}.
Starting from the anchor-induced initialisation $\mathbf{z}^{(0)}\!=\!\mathcal{H}_{\mathbf{x}}(\mathbf{W}^{(0)})$ where $\mathbf{w}_c^{(0)}\!=\!\mathrm{Norm}(f_{\mathcal T}(\mathrm{Template}(c)))$ as in Section~\ref{sec:clbp}, define the \emph{reachable set} as the smallest closed, $\mathcal{T}_{\mathbf{x}}$-invariant subset of $\mathbb{S}^{d-1}$ that contains $\mathbf{z}^{(0)}$:
\begin{equation}
  \Omega
  \;\triangleq\;
  \overline{\left\{\mathbf{z}^{(k)}:\ \mathbf{z}^{(k+1)}=\mathcal{T}_{\mathbf{x}}(\mathbf{z}^{(k)}),\ k\ge 0\right\}},
  \qquad
  \mathcal{T}_{\mathbf{x}}(\Omega)\subseteq \Omega.
  \label{eq:Omega}
\end{equation}
Since $\Omega$ is closed in $\mathbb{S}^{d-1}\!\subset\!\mathbb{R}^d$, the pair $(\Omega,\|\cdot\|_2)$ is a complete metric space.
Our analysis uses one regularity quantity, the local Lipschitz constant of the composed operator,
$q \!\triangleq\! \mathrm{Lip}_\Omega(\mathcal{T}_{\mathbf{x}}) \!=\! \sup_{\mathbf{z}_1\neq\mathbf{z}_2\in\Omega} \|\mathcal{T}_{\mathbf{x}}(\mathbf{z}_1)\!-\!\mathcal{T}_{\mathbf{x}}(\mathbf{z}_2)\|_2 / \|\mathbf{z}_1\!-\!\mathbf{z}_2\|_2$,
and the local-contraction condition $q\!<\!1$.

\paragraph{Factorised upper bound (loose).}
A useful but \emph{not} sharp bound follows from per-branch Lipschitz constants:
\begin{align}
  \|\mathcal{G}(\mathbf{z}_1)-\mathcal{G}(\mathbf{z}_2)\|_F
  &\le L_G \|\mathbf{z}_1-\mathbf{z}_2\|_2,
  &&\forall\, \mathbf{z}_1,\mathbf{z}_2\!\in\!\Omega,
  \label{eq:LG_assump}\\
  \|\mathcal{H}_{\mathbf{x}}(\mathbf{W}_1)-\mathcal{H}_{\mathbf{x}}(\mathbf{W}_2)\|_2
  &\le L_H \|\mathbf{W}_1-\mathbf{W}_2\|_F,
  &&\forall\, \mathbf{W}_1,\mathbf{W}_2\!\in\!\mathcal{G}(\Omega),
  \label{eq:LH_assump}
\end{align}
which compose to $q\!\le\!L_G L_H$.
This factorised version is convenient for the proofs below but is generally loose for CLBP, because the V2T branch updates only a few prompt tokens through an anchor-constrained $\mathrm{Compose}(\cdot)$ bottleneck; the loop-direction sensitivity is much smaller than $L_G L_H$ (Tab.~\ref{tab:empirical_lipschitz}).

\paragraph{Why this is consistent with CLBP.}
The local contraction is structurally plausible: anchor initialisation pins the iterates near pretrained semantics, unit-norm projection bounds the visual side to a compact set, and the lightweight adapters (built from linear layers, attention, normalisation, and pointwise activations) are locally Lipschitz on any compact set.
Empirical verification follows in Appendix~\ref{app:empirical_lipschitz}.

\subsection{Empirical Verification of the Local Contraction}
\label{app:empirical_lipschitz}

We measure the per-branch Lipschitz constants $L_G,L_H$ and a finite-difference estimate of the composed constant $q$ on held-out clean and adversarial trajectories of the trained CLBP model (DTD, ViT-B/32; adversarial samples generated with $\epsilon{=}1/255$).
For each pair of consecutive iterates we record $\widehat{q}=\|\mathcal{T}_{\mathbf{x}}(\mathbf{z}_1)-\mathcal{T}_{\mathbf{x}}(\mathbf{z}_2)\|_2/\|\mathbf{z}_1-\mathbf{z}_2\|_2$, $\widehat{L}_G=\|\mathcal{G}(\mathbf{z}_1)-\mathcal{G}(\mathbf{z}_2)\|_F/\|\mathbf{z}_1-\mathbf{z}_2\|_2$, and $\widehat{L}_H=\|\mathcal{H}_{\mathbf{x}}(\mathbf{W}_1)-\mathcal{H}_{\mathbf{x}}(\mathbf{W}_2)\|_2/\|\mathbf{W}_1-\mathbf{W}_2\|_F$, and aggregate over the trajectory ensemble.

\begin{table}[h!]
  \centering
  \caption{\textbf{Empirical Lipschitz constants of the closed loop on DTD (ViT-B/32).} Finite-difference estimates over inference trajectories of the 16-shot CLBP model (trained 10 epochs, 2-step PGD at $\epsilon_{\mathrm{train}}{=}1/255$). $\widehat{q}$ is the local Lipschitz constant of the composition $\mathcal{T}_{\mathbf{x}}{=}\mathcal{H}_{\mathbf{x}}{\circ}\mathcal{G}$ measured at the inferred state; $\widehat{L}_G,\widehat{L}_H$ are the per-branch constants of $\mathcal{G}$ and $\mathcal{H}_{\mathbf{x}}$; $\widehat{q}_{GH}\!\triangleq\!\widehat{L}_G\widehat{L}_H$ is the factorised proxy used in Section~\ref{sec:theory_main}. Adversarial trajectories are obtained under PGD-100 at $\epsilon_a{=}1/255$.}
  \label{tab:empirical_lipschitz}
  \small
  \setlength{\tabcolsep}{6pt}
  \begin{tabular}{l c c c c c c c}
    \toprule
    Setting     & mean $\widehat{q}$ & median $\widehat{q}$ & $\%(\widehat{q}\!<\!1)$ & mean $\widehat{q}_{GH}$ & median $\widehat{q}_{GH}$ & mean $\widehat{L}_G$ & mean $\widehat{L}_H$ \\
    \midrule
    Clean       & $8.39{\times}10^{-5}$ & $8.22{\times}10^{-5}$ & $100.0$ & $1.46{\times}10^{-4}$ & $1.20{\times}10^{-4}$ & $1.393$ & $1.033$ \\
    Adversarial & $9.73{\times}10^{-5}$ & $9.57{\times}10^{-5}$ & $100.0$ & $1.46{\times}10^{-4}$ & $1.20{\times}10^{-4}$ & $1.297$ & $1.105$ \\
    \bottomrule
  \end{tabular}
\end{table}

The composed estimate $\widehat{q}$ is $\mathcal{O}(10^{-4})$ and stays below $1$ on every evaluated trajectory.
The factorised proxy $\widehat{q}_{GH}$ is also below $1$, so Theorem~\ref{thm:conv_main} applies through Eqs.~\eqref{eq:LG_assump}--\eqref{eq:LH_assump}.
The per-branch estimates $\widehat{L}_G,\widehat{L}_H$ can exceed $1$ on some pairs, which is consistent with the V2T branch using nonlinear cross-attention; the contraction is therefore a property of the composition rather than of either branch individually.
This is local empirical evidence that the contraction condition holds on the reachable region used by the trained model, not a global Lipschitz certificate.

\paragraph{Fixed-point proximity along the closed loop.}
To verify that a single update is sufficient in practice, we track $d_k\!=\!\|\mathbf{z}^{(k)}-\mathbf{z}^\star\|_2$ and the corresponding logit drift along inference trajectories on DTD, taking $\mathbf{z}^\star$ as the iterate at convergence.

\begin{table}[h!]
  \centering
  \caption{\textbf{Closed-loop convergence to the empirical fixed point on DTD.} 16-shot CLBP (10 epochs, 2-step PGD at $\epsilon_{\mathrm{train}}{=}1/255$) evaluated under PGD-100 at $\epsilon_a{=}1/255$. We report the distance $d_k\!=\!\|\mathbf{z}^{(k)}{-}\mathbf{z}^\star\|_2$ to the empirical fixed point and Top-1 classification accuracy (\%) after $k$ closed-loop updates; $k{=}0$ corresponds to the anchor initialisation.}
  \label{tab:fixed_point_proximity}
  \small
  \setlength{\tabcolsep}{12pt}
  \begin{tabular}{c c c c}
    \toprule
    Step $k$ & mean $d_k$ & mean logit drift & Accuracy (\%) \\
    \midrule
    $0$ & $1.91{\times}10^{-6}$ & $3.67{\times}10^{1}$  & $38.50$ \\
    $1$ & $\mathbf{8.89{\times}10^{-7}}$ & $\mathbf{5.61{\times}10^{-5}}$ & $\mathbf{54.50}$ \\
    $2$ & $8.54{\times}10^{-7}$ & $5.19{\times}10^{-5}$ & $54.50$ \\
    \bottomrule
  \end{tabular}
\end{table}

From $k{=}0$ to $k{=}1$ the accuracy moves by $16\%$ and the logit drift drops by six orders of magnitude; for $k{\geq}1$, $d_k$ stays at $\sim\!10^{-7}$ and accuracy is unchanged.
This matches the near-fixed-point regime of Theorem~\ref{thm:conv_main} with $q\!\sim\!10^{-4}$, and is why we use $N{=}1$ in the main paper.

\subsection{Full Proof of Theorem~\ref{thm:conv_main}}
\label{app:convproof}

\begin{theorem}[Convergence, restated]
\label{thm:conv_main_app}
If $q\!\triangleq\!\mathrm{Lip}_\Omega(\mathcal{T}_{\mathbf{x}})<1$, then $\mathcal{T}_{\mathbf{x}}$ is a contraction on $\Omega$ and admits a unique fixed point $\mathbf{z}^\star\!\in\!\Omega$.
Moreover, $\|\mathbf{z}^{(k)}-\mathbf{z}^\star\|_2\le q^k\|\mathbf{z}^{(0)}-\mathbf{z}^\star\|_2$.
\end{theorem}

\begin{proof}
By the definition $q\!=\!\mathrm{Lip}_\Omega(\mathcal{T}_{\mathbf{x}})$ from Appendix~\ref{app:setup}, the inequality
\begin{equation}
  \|\mathcal{T}_{\mathbf{x}}(\mathbf{z}_1)-\mathcal{T}_{\mathbf{x}}(\mathbf{z}_2)\|_2
  \;\le\;
  q\,\|\mathbf{z}_1-\mathbf{z}_2\|_2,
  \qquad \forall\,\mathbf{z}_1,\mathbf{z}_2\!\in\!\Omega,
  \label{eq:contract}
\end{equation}
holds, and establishes the contraction property whenever $q\!<\!1$.
The factorised sufficient form $q\!\le\!L_GL_H$ obtained by chaining $\mathcal{T}_{\mathbf{x}}\!=\!\mathcal{H}_{\mathbf{x}}\!\circ\!\mathcal{G}$ through Eqs.~\eqref{eq:LG_assump}--\eqref{eq:LH_assump} therefore implies $q\!<\!1$ a fortiori; this factorised version is loose for CLBP (cf.~Tab.~\ref{tab:empirical_lipschitz}) but it is what we invoke when needed in the rest of the proofs.
Since $\Omega$ is complete (Eq.~\eqref{eq:Omega}) and $\mathcal{T}_{\mathbf{x}}(\Omega)\!\subseteq\!\Omega$, the Banach fixed-point theorem applied to the contraction~\eqref{eq:contract} yields a unique fixed point $\mathbf{z}^\star\!=\!\mathcal{T}_{\mathbf{x}}(\mathbf{z}^\star)$ with the geometric rate $\|\mathbf{z}^{(k)}-\mathbf{z}^\star\|_2 \le q^k\|\mathbf{z}^{(0)}-\mathbf{z}^\star\|_2$.
\end{proof}

\subsection{Full Proof of Theorem~\ref{thm:stab_main}}
\label{app:stabproof}

We first state a standard perturbation lemma for contraction mappings.

\begin{lemma}[Fixed-point perturbation bound]
\label{lem:fp_perturb}
Let $(\Omega,\|\cdot\|)$ be a complete metric space.
Let $T,S:\Omega\rightarrow\Omega$ be contractions with the same factor $q<1$:
$\|T(u)-T(v)\|\le q\|u-v\|$ and $\|S(u)-S(v)\|\le q\|u-v\|$.
Let $u^\star$ and $v^\star$ be the unique fixed points of $T$ and $S$.
If $\sup_{u\in\Omega}\|T(u)-S(u)\|\le \eta$, then
\begin{equation}
\|u^\star-v^\star\| \le \frac{\eta}{1-q}.
\label{eq:fp_perturb_bound}
\end{equation}
\end{lemma}

\begin{proof}
Using $u^\star=T(u^\star)$ and $v^\star=S(v^\star)$, we have
\begin{align}
\|u^\star-v^\star\|
&=
\|T(u^\star)-S(v^\star)\| \nonumber\\
&\le
\|T(u^\star)-T(v^\star)\|
+
\|T(v^\star)-S(v^\star)\| \nonumber\\
&\le
q\|u^\star-v^\star\| + \eta.
\end{align}
Rearranging gives $(1-q)\|u^\star-v^\star\|\le\eta$, proving \eqref{eq:fp_perturb_bound}.
\end{proof}

We now prove Theorem~\ref{thm:stab_main} in Section~\ref{sec:theory_main}.

\begin{theorem}[Stability under adversarial perturbations]
\label{thm:stab_main_app}
Assume $\|\mathcal{H}_{\mathbf{x}+\boldsymbol{\delta}}(\mathbf{W})-\mathcal{H}_{\mathbf{x}}(\mathbf{W})\|_2\le L_x\|\boldsymbol{\delta}\|$ for all $\mathbf{W}\in\mathcal{G}(\Omega)$.
Let $\mathbf{z}^\star(\mathbf{x})$ and $\mathbf{z}^\star(\mathbf{x}+\boldsymbol{\delta})$ be the fixed points of $\mathcal{T}_{\mathbf{x}}$ and $\mathcal{T}_{\mathbf{x}+\boldsymbol{\delta}}$, respectively.
Then
\begin{equation}
\big\|\mathbf{z}^\star(\mathbf{x}+\boldsymbol{\delta})-\mathbf{z}^\star(\mathbf{x})\big\|_2
\le
\frac{L_x}{1-q}\|\boldsymbol{\delta}\|.
\end{equation}
\end{theorem}

\begin{proof}
Define $T=\mathcal{T}_{\mathbf{x}+\boldsymbol{\delta}}$ and $S=\mathcal{T}_{\mathbf{x}}$ on $\Omega$.
Under the assumptions of Theorem~\ref{thm:conv_main_app}, both $T$ and $S$ are contractions on $\Omega$ with factor $q<1$.

It remains to bound $\sup_{\mathbf{z}\in\Omega}\|T(\mathbf{z})-S(\mathbf{z})\|_2$.
For any $\mathbf{z}\in\Omega$,
\begin{align}
\|T(\mathbf{z})-S(\mathbf{z})\|_2
&=
\|\mathcal{H}_{\mathbf{x}+\boldsymbol{\delta}}(\mathcal{G}(\mathbf{z}))-\mathcal{H}_{\mathbf{x}}(\mathcal{G}(\mathbf{z}))\|_2 \nonumber\\
&\le
L_x \|\boldsymbol{\delta}\|
\qquad \text{(by the assumed input Lipschitzness).}
\end{align}
Thus we can take $\eta=L_x\|\boldsymbol{\delta}\|$ in Lemma~\ref{lem:fp_perturb}, which yields
\begin{equation}
\big\|\mathbf{z}^\star(\mathbf{x}+\boldsymbol{\delta})-\mathbf{z}^\star(\mathbf{x})\big\|_2
\le
\frac{L_x}{1-q}\|\boldsymbol{\delta}\|.
\end{equation}
This completes the proof.
\end{proof}


\section{Additional Lemmas}
\label{app:morelemmas}

\subsection{Margin Implies Prediction Invariance (Feature Space)}
\label{app:margin}

\begin{lemma}[Feature-space margin implies invariant prediction]
\label{lem:margin}
Let $\{\mathbf{w}_c\}_{c=1}^C\subset\mathbb{S}^{d-1}$ be normalized class prototypes and $\mathbf{z}\in\mathbb{S}^{d-1}$ be a normalized image feature.
Define cosine scores $s_c=\langle \mathbf{z},\mathbf{w}_c\rangle$ and the predicted label $y=\arg\max_c s_c$.
If the margin condition holds:
\begin{equation}
s_y - \max_{c\neq y} s_c \ge \gamma,
\label{eq:margin_cond_app}
\end{equation}
then for any perturbation $\Delta\mathbf{z}$ with $\|\Delta\mathbf{z}\|_2\le \gamma/2$, the prediction does not change:
\begin{equation}
\arg\max_c \langle \mathbf{z}+\Delta\mathbf{z}, \mathbf{w}_c\rangle = y.
\end{equation}
\end{lemma}

\begin{proof}
By Cauchy--Schwarz, $|\langle \Delta\mathbf{z},\mathbf{w}_c\rangle|\!\le\!\|\Delta\mathbf{z}\|_2$ for every $c$ since $\|\mathbf{w}_c\|_2\!=\!1$, so each cosine score shifts by at most $\|\Delta\mathbf{z}\|_2$. Letting $c^\star\!=\!\arg\max_{c\neq y}s_c$,
\begin{equation}
  \langle \mathbf{z}+\Delta\mathbf{z},\mathbf{w}_y\rangle-\langle \mathbf{z}+\Delta\mathbf{z},\mathbf{w}_{c^\star}\rangle
  \;\ge\; (s_y-s_{c^\star})-2\|\Delta\mathbf{z}\|_2
  \;\ge\; \gamma-2\|\Delta\mathbf{z}\|_2,
\end{equation}
which is non-negative whenever $\|\Delta\mathbf{z}\|_2\!\le\!\gamma/2$, so $y$ remains the maximiser.
\end{proof}

\paragraph{Connection to CLBP.}
By Theorem~\ref{thm:stab_main_app}, CLBP bounds the feature deviation induced by an input perturbation.
Combining it with Lemma~\ref{lem:margin} yields a sufficient condition for prediction invariance under bounded perturbations.

\subsection{Outlier Suppression of Multi-view Weights}
\label{app:mv_outlier}

We analyze the aggregation in Eq.~\eqref{eq:mv_main}.
For each view $i$, the weight is
\begin{equation}
\alpha_i=\frac{\exp(\mathrm{score}_i/\tau_{\mathrm{mv}})}{\sum_{r=1}^V \exp(\mathrm{score}_r/\tau_{\mathrm{mv}})},
\qquad
\mathrm{score}_i=\frac{1}{k}\sum_{j\in \mathrm{TopK}_k(\mathbf{S}_{i,:})}\mathbf{S}_{ij},
\qquad
\mathbf{S}_{ij}=\langle \mathbf{z}_i,\mathbf{z}_j\rangle.
\end{equation}

\begin{theorem}[Outlier suppression by similarity-softmax weighting]
\label{thm:outlier}
Partition the $V$ views into a consistent set $\mathcal{G}$ with $|\mathcal{G}|=M$ and an outlier set $\mathcal{B}$ with $|\mathcal{B}|=V-M$.
Assume there exist $a>b$ such that
\begin{equation}
\mathrm{score}_i \ge a\quad \forall i\in\mathcal{G},
\qquad
\mathrm{score}_j \le b\quad \forall j\in\mathcal{B}.
\label{eq:ab_assump}
\end{equation}
Then the total outlier weight is bounded by
\begin{equation}
\sum_{j\in\mathcal{B}}\alpha_j
\le
\frac{V-M}{M\exp((a-b)/\tau_{\mathrm{mv}})+V-M}.
\label{eq:outlier_bound}
\end{equation}
\end{theorem}

\begin{proof}
The numerator of $\sum_{j\in\mathcal{B}}\alpha_j$ is bounded above by $(V\!-\!M)e^{b/\tau_{\mathrm{mv}}}$ via $\mathrm{score}_j\!\le\!b$, and the denominator below by $M e^{a/\tau_{\mathrm{mv}}}\!+\!(V\!-\!M)e^{b/\tau_{\mathrm{mv}}}$ via $\mathrm{score}_i\!\ge\!a$. Combining and dividing through by $e^{b/\tau_{\mathrm{mv}}}$,
\begin{equation}
\sum_{j\in\mathcal{B}}\alpha_j
\;\le\;
\frac{(V-M)\,e^{b/\tau_{\mathrm{mv}}}}{M e^{a/\tau_{\mathrm{mv}}}+(V-M)e^{b/\tau_{\mathrm{mv}}}}
\;=\;
\frac{V-M}{M e^{(a-b)/\tau_{\mathrm{mv}}}+(V-M)},
\end{equation}
which proves \eqref{eq:outlier_bound}.
\end{proof}

\paragraph{Interpretation.}
When $\tau_{\mathrm{mv}}$ is small (sharp softmax), any non-trivial gap $a-b$ exponentially suppresses outlier weights,
formalizing the robustness benefit of multi-view consensus.

\section{Broader Impact}
\label{sec:broader_impact}

CLBP addresses a security concern that grows in importance as Vision Language Models are deployed in safety-relevant settings. By improving robustness without retraining the underlying encoders, our method offers a parameter-efficient route to mitigating adversarial vulnerabilities that may be especially useful for resource-constrained or already-deployed systems.

We acknowledge two considerations. \emph{Dual-use nature.} Defensive techniques in adversarial machine learning are dual-use: the same observation that motivates our defence---that cross-modal alignment can be disrupted bidirectionally---could in principle inform stronger attacks. We mitigate this by evaluating only against well-established benchmarks (PGD, AutoAttack, CW, EOT-PGD) and by releasing our code openly so that the community can audit and improve both defences and threat models. \emph{Inherited biases.} Like any test-time defence built on a frozen backbone, CLBP relies on the quality of the pre-trained text encoder; biases or stereotypes present in the underlying VLM are not corrected by our method and may persist in downstream outputs. We therefore recommend that practitioners pair adversarial defences such as CLBP with bias and fairness audits appropriate to their deployment context.

To support transparent evaluation and reproducibility, our implementation, training scripts, and pre-trained adapters are released under a permissive open-source license, and the experimental protocols are documented in Section~\ref{sec:experiments} and Appendices~\ref{app:algo}--\ref{app:eot}.


\end{document}